%% file: main.tex
\documentclass[Afour,sageh,times]{sagej}

\input{include/packages}

\input{include/macros}

\newcommand\BibTeX{{\rmfamily B\kern-.05em \textsc{i\kern-.025em b}\kern-.08em
T\kern-.1667em\lower.7ex\hbox{E}\kern-.125emX}}

\def\volumeyear{2026}

\begin{document}

\runninghead{Cao et al.}

\title{HIGS: Hierarchical Implicit Grids for Joint Geometric and Semantic Scene Understanding}

\author{%
Hanwen Cao\affilnum{1$\dagger$},
Wenqiang Wu\affilnum{1,4*},
Kuang-Ting Tu\affilnum{1*},
Mathias Otnes\affilnum{1,5*},
Jeffrey Delmerico\affilnum{2}, Rui Wang\affilnum{2}, Yulun Tian\affilnum{3}
and Nikolay Atanasov\affilnum{1}
}

\affiliation{%
\affilnum{1}
University of California San Diego,
La Jolla, CA, USA.
\affilnum{2}
Microsoft Spatial AI Lab, Z\"{u}rich, Switzerland.
\affilnum{3}
University of Michigan,
Ann Arbor, MI, USA.
\affilnum{4}
Southern University of Science and Technology,
Shenzhen, Guangzhou, China.
\affilnum{5}
Norwegian University of Science and Technology,
Trondheim, Norway.
\affilnum{$\dagger$} Work done partially during an internship at Microsoft Research.
\affilnum{*}
Equal second author.
}

\corrauth{%
Hanwen Cao,
University of California San Diego,
La Jolla, CA, USA
}

\email{h1cao@ucsd.edu}

\begin{abstract}
Neural implicit representations have had a significant impact on scene reconstruction by enabling robots to build continuous, differentiable, and high-fidelity 3D maps. Most existing works focus on geometric reconstruction and lack semantic information for high-level spatial understanding and task planning. Also, as the scale and complexity of the environment increase, neural representations face the challenge of maintaining computational efficiency in back-end optimization. To resolve these two challenges, we introduce a hierarchical neural field that leverages multiresolution submaps to achieve an efficient and scalable implicit representation, and a unified query and decoding mechanism to support both geometric and semantic features. More specifically, the learnable map features can be converted to the output with the query and decoding process for both training and inference. For large-scale representation, we decompose a scene into overlapping submaps and do hierarchical optimization within each local submap, thus enabling scalable computation. To further improve efficiency, we design feature encoders that predict initial hierarchical grid features to substantially reduce the time needed to optimize the submap features from scratch. To correct estimation drift among submaps, we align and fuse them entirely within the implicit feature space, leading to substantial acceleration by avoiding the need to decode the final output. Building upon this efficient hierarchical representation, we embed both geometric features and vision-language latent features into the map, and demonstrate it on both Signed Distance Field (SDF) construction and open-vocabulary object grounding. Our approach significantly improves computation and memory efficiency, maintains high estimation accuracy, and endows the robot with spatial awareness on large-scale real-world benchmarks. We call our method \AlgName for \underline{H}ierarchical \underline{I}mplicit \underline{G}rid\underline{S}, with GS also emphasizing joint \underline{G}eometric and \underline{S}emantic scene understanding. Our project webpage is at: \url{https://existentialrobotics.org/HIGS_webpage}.
\end{abstract}


\maketitle

\begin{figure*}[!t]
    \centering
    \includegraphics[
        width=\textwidth
    ]{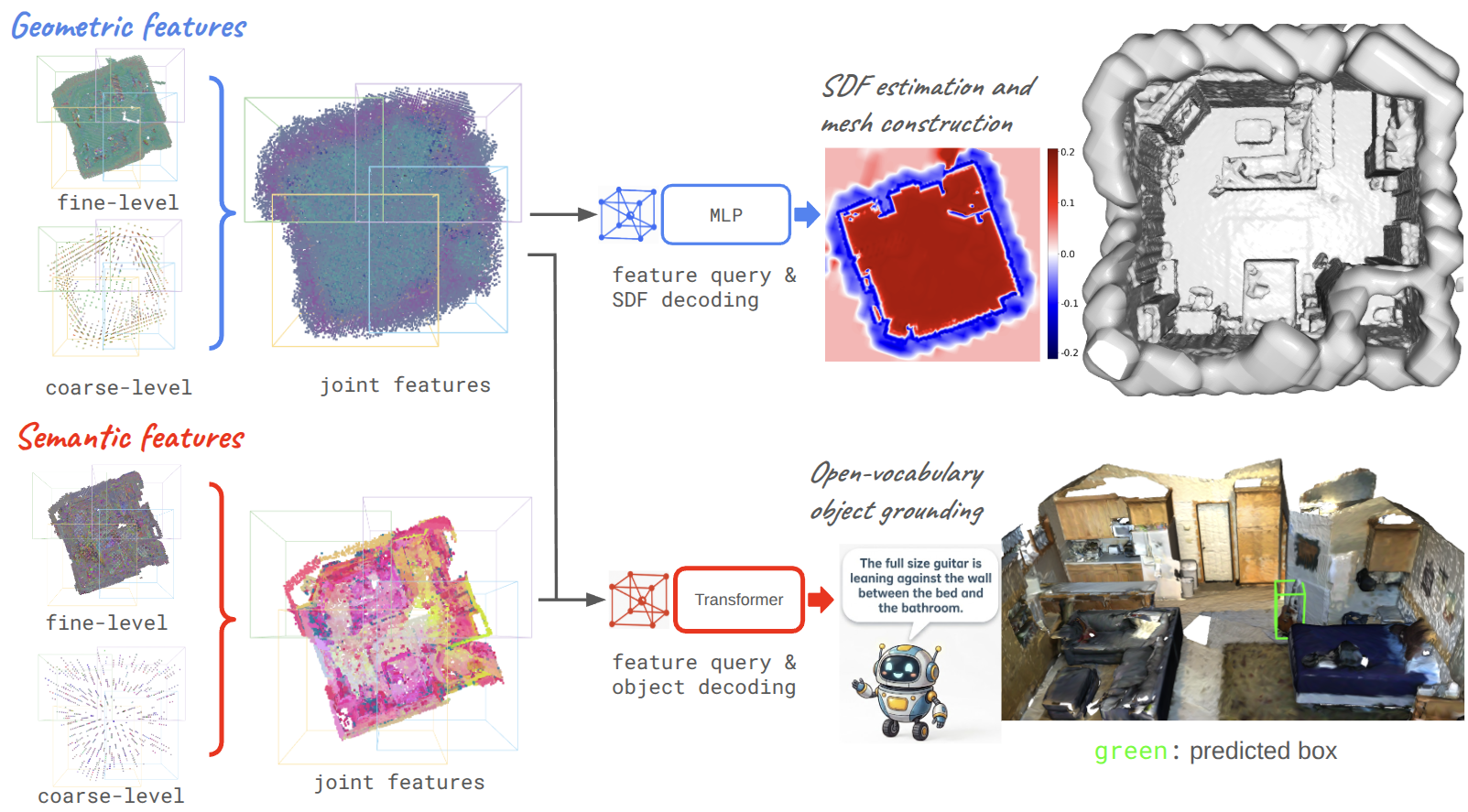}
    \caption{Visualization of the architecture of \AlgName, applied to an example from the ScanNet dataset~\citep{dai2017scannet}. The left part visualizes the hierarchical grids of implicit features using principal component analysis. The features are used for both signed distance function estimation and open-vocabulary object grounding.}
    \label{fig:teaser}
\end{figure*}

\setcounter{footnote}{0}
\renewcommand{\thefootnote}{\arabic{footnote}}

\input{sections/introduction}

\input{sections/related_work}

\input{sections/overview}

\input{sections/method}

\input{sections/experiments}

\input{sections/conclusion}

\bibliographystyle{SageH} 
\bibliography{main}

\begin{appendices}
    \input{sections/appendix_proof}
    \input{sections/appendix_details}
\end{appendices}

\end{document}

%% file: include/packages.tex
\usepackage{appendix}
\usepackage{moreverb,url}

\usepackage[colorlinks,bookmarksopen,bookmarksnumbered,citecolor=red,urlcolor=red]{hyperref}

\usepackage[table,dvipsnames]{xcolor}      
\usepackage{extarrows}                              
\usepackage{enumitem}
\usepackage{microtype}
\usepackage{cleveref}

\usepackage{amsmath,amssymb,amsfonts,amsthm,dsfont} 
\usepackage{algorithm,algorithmicx,listings}        
\usepackage[noend]{algpseudocode}			              
\usepackage{mathtools}

\usepackage{caption}
\usepackage{graphicx,tabularx,subcaption,booktabs,adjustbox}
\usepackage{pgfplots}
\pgfplotsset{compat=1.15}
\usepackage{tabularx, adjustbox, multirow}

\usepackage{xspace}
\usepackage[normalem]{ulem}
\useunder{\uline}{\ul}{}

\newtheorem{proposition}{Proposition}

\theoremstyle{definition}
\newtheorem{definition}{Definition}
\newtheorem{problem}{Problem}

%% file: include/macros.tex
\input{include/sym}

\newcommand{\eg}{\textit{e.g.}\xspace}
\newcommand{\ie}{\textit{i.e.}\xspace}
\newcommand{\etc}{\textit{etc.}\xspace}

\providecommand{\edit}[1]{{#1}}

\newcommand{\lidar}{{LiDAR}\xspace}
\newcommand{\AlgName}{{HIGS}\xspace}
\newcommand{\Point}{{Neural Points}\xspace}
\newcommand{\iSDF}{{iSDF}\xspace}

\newcommand{\cfeat}{c^\text{feat}}

\newcommand{\myParagraph}[1]{{\bf #1.}\xspace}

\newcommand{\Real}{\mathbb{R}}

\def \pinv {^\dagger}

\providecommand{\method}[1]{{\small \textsf{#1}}\xspace}

\DeclareMathOperator{\SE}{SE}

\DeclareMathOperator{\Exp}{Exp}
\DeclareMathOperator{\Log}{Log}

\providecommand{\norm}[1]{\left\|#1\right\|}

\newcommand{\Ecal}{\mathcal{E}}

\newcommand{\That}{\widehat{T}}

%% file: include/sym.tex
\newcommand{\bfc}{\mathbf{c}}
\newcommand{\bfd}{\mathbf{d}}

\newcommand{\bff}{\mathbf{f}}
\newcommand{\bfg}{\mathbf{g}}

\newcommand{\bfi}{\mathbf{i}}

\newcommand{\bfp}{\mathbf{p}}

\newcommand{\bft}{\mathbf{t}}

\newcommand{\bfv}{\mathbf{v}}

\newcommand{\bfx}{\mathbf{x}}
\newcommand{\bfy}{\mathbf{y}}
\newcommand{\bfz}{\mathbf{z}}

\newcommand{\bbR}{\mathbb{R}}

\newcommand{\calO}{\mathcal{O}}



%% file: sections/introduction.tex
\section{Introduction}


In recent years, \textit{neural fields} \citep{xie2022neural} have emerged as a powerful new frontier for scene representation in simultaneous localization and mapping (SLAM). Compared to conventional approaches based on hand-crafted features or volumetric representations, neural representations \citep{tosi2024nerfs} offer continuous and differentiable scene modeling, improved memory efficiency, and better handling of measurement noise. Despite these advantages, most existing approaches focus exclusively on geometric reconstruction, e.g., signed distance fields (SDF)~\citep{ortiz2022isdf}, neural radiance fields~\citep{zhu2022nice}, and \etc. While extracting high-fidelity geometry is crucial for low-level tasks like obstacle avoidance and navigation, high-level tasks like object search and physical interaction require semantic and language understanding.
Moreover, as the environment size and mission time grow, most existing approaches~\citep{zhu2022nice,tie20242} consider increasingly larger back-end optimization problems, which ultimately causes computational bottlenecks that limit their real-time performance.

A powerful idea to achieve a highly efficient map representation is to rely on a hierarchical data structure that explicitly disentangles coarse and fine information. Equipped with a hierarchical model, a robot can perform inference over varying spatial resolutions, e.g., by first capturing the core structure in the environment and then optimizing the fine details. In SLAM, this idea dates back to several seminal works such as \citet{estrada2005hierarchical,frese2005multilevel,grisetti2010hierarchical}. While hierarchical representations have achieved state-of-the-art performance in various computer vision tasks \citep{takikawa2021nglod,sun2022direct,muller2022instant}, standard neural SLAM back-ends rarely exploit them for feature embedding or accelerated optimization.

In this work, we present a unified formulation that utilizes \textit{multiresolution grids} and \textit{hierarchical optimization} to construct spatial models of geometric and semantic structure. Our formulation allows performing a significant portion of the back-end optimization in the \textit{implicit feature space}, yielding substantial gains in efficiency compared to existing methods \citep{tang2023mips,zhai2024vox}. To scale to larger environments, we adopt a submap-based design that models the environment as a collection of local neural implicit maps. In this context, we show that the proposed hierarchical optimization in implicit feature space significantly enhances both \textit{local submap optimization} and \textit{global submap fusion}.

We call our method \AlgName (\underline{H}ierarchical \underline{I}mplicit \underline{G}rid\underline{S}), and demonstrate that it achieves joint geometric and semantic understanding by applying it to neural signed distance field (SDF) construction \citep{ortiz2022isdf} and open-vocabulary object grounding~\citep{mcvay2025locate}. For the object grounding task, \AlgName serves as a memory-efficient representation and can be used by existing neural networks, e.g., \citet{mcvay2025locate} via point feature query. 
We also design a neural network that makes use of the geometric features to obtain fine-grained structure information and improve the object grounding accuracy with only semantic features, showing the benefit of joint embedding. We demonstrate our method's effectiveness on large-scale real-world datasets. An illustration of our method can be found in \Cref{fig:teaser}. Our contributions are summarized below.

\begin{itemize}
    \item We introduce \AlgName, a unified hierarchical 3D representation and optimization approach for joint geometric and semantic feature embedding. We improve the computation and memory efficiency via different techniques, and explore the synergy between the geometric and semantic features.
    
    \item For scalable and efficient feature mapping, we design multiple techniques, including pretrained decoders for hierarchical feature query, pre-trained encoders for latent grid feature initialization, submap decomposition, and hierarchical optimization in latent feature space for submap alignment.

    \item We show that geometric features and semantic features can enhance each other by designing an attention block integrating SDF features that improves object grounding accuracy and demonstrating that semantic features can help improve the submap alignment. 
    
    \item Experiments in large-scale real-world benchmarks are conducted to demonstrate the effectiveness of our method in both geometric construction (SDF) and semantic understanding (object grounding). For SDF construction, we show that \AlgName can get better construction accuracy with significant speedup (up to $40\times$). For object grounding, we show that \AlgName can significantly compress the memory (over 20$\times$) with small accuracy loss ($2\% - 3\%$), and increase the grounding accuracy up to $15$ percentage points by leveraging the SDF features. 
\end{itemize}


%% file: sections/related_work.tex
\section{Related work}

\subsection{3D neural implicit representations} 

Neural implicit representations offer continuous and differentiable modeling of 3D scenes with high fidelity, memory efficiency, and robustness to noise \citep{park2019deepsdf, mescheder2019occupancy, mildenhall2021nerf, azinovic2022neural}. Early methods such as DeepSDF \citep{park2019deepsdf} and NeRF \citep{mildenhall2021nerf} rely solely on 3D coordinates and a single multi-layer perceptron (MLP) to reconstruct the scene. However, this approach is insufficient for capturing larger scenes or complex details, prompting subsequent works to introduce hybrid methods that combine MLP decoders with additional implicit features.
The implicit features are commonly organized in a 3D grid \citep{fridovich2022plenoxels, sun2022direct,takikawa2021nglod,muller2022instant}.
To enable continuous scene modeling, trilinear interpolation is used to infer a feature at an arbitrary query location that is subsequently passed through the MLP decoder to predict the environment model (\eg, occupancy, distance, radiance). Recent works propose several alternative approaches to improve the memory efficiency over 3D feature grids.
K-Planes \citep{fridovich2023k} factorizes the scene representation into multiple 2D feature planes rather than using a full 3D grid. Similarly, TensoRF~\citep{chen2022tensorf} employs tensor decomposition to compactly represent radiance fields. 
PointNeRF~\citep{xu2022point} constructs the scene representation directly from point clouds by efficiently aggregating local features at surface points.
Hierarchical strategies for organizing the implicit features have been particularly effective at capturing different levels of detail while maintaining efficiency \citep{muller2022instant, sun2022direct, li2023neuralangelo, yu2021plenoctrees, takikawa2021nglod}. 
DVGO~\citep{sun2022direct} performs progressive scaling that gradually increases the feature grid resolution during training.
InstantNGP \citep{muller2022instant} significantly accelerates feature grid training and inference by introducing a multiresolution hash encoding scheme.
Neuralangelo \citep{li2023neuralangelo} extends this concept with a coarse-to-fine optimization scheme that preserves fine-grained details.  
Hierarchical representations have also been explored for fast RGB-D surface reconstruction \citep{wang2022go, azinovic2022neural}. 
Despite these advancements, achieving high-fidelity, efficient, and globally consistent reconstruction of large-scale scenes remains challenging. 
Our preceding work~\citep{tian2025miso} decomposes a scene into submaps, with each submap represented as hierarchical 3D grids, and designs hierarchical initialization and optimization to improve efficiency. 
While most existing works focus on geometric construction, in this work, we design hierarchical implicit grids as general feature fields including geometry, semantic, and language features, and demonstrate that this multi-modal representation delivers synergies in both geometric reconstruction and object grounding. 

\subsection{Scene understanding}

Language-aware 3D representations enable robots to connect natural-language queries with their surroundings.
OpenScene~\citep{peng2023openscene} aligns dense point features with vision-language embeddings, enabling open-vocabulary semantic retrieval through text–feature similarity. However, this query mechanism does not explicitly model inter-object relations for instance disambiguation. ConceptGraphs~\citep{gu2024conceptgraphs} organizes object point clouds and semantic descriptors into an open-vocabulary scene graph, supporting relational reasoning and language-based planning. Its object-level semantic abstraction differs from a continuous field supporting spatially resolved queries.
Locate-3D~\citep{mcvay2025locate} represents a scene with a point cloud created from posed RGB-D frames. To incorporate semantic and language features, they run CLIP~\citep{clip} and DINOv2~\citep{oquab2024dinov2} on 2D images and project them to corresponding 3D points. The featurized point cloud is fed into a transformer \citep{wu2024point} to capture spatial relations for accurate 3-D bounding box prediction. 
UniVLG~\citep{jain2025unifying} unifies 2D and 3D vision-language understanding by leveraging a pretrained 2D visual backbone and jointly training on RGB and posed RGB-D data. It introduces a language-conditioned mask decoder that directly grounds referring expressions to 3D object masks.
PQ3D \citep{zhu2024unifying} uses pre-computed object masks to create an object map with different features \citep{qi2017pointnet++,ghiasi2022scaling,schult2023mask3d} and treats the object grounding as a classification task. Qwen-3D \citep{lin2026qwen} extends large multimodal models to 3D by incorporating multi-view geometry and 3D rotary positional embeddings, allowing visual tokens to interact directly in a shared 3D coordinate space. 
Compared to existing work, our method focuses on efficient joint geometric and semantic scene representation. While existing works~\citep{zhu2024sni,zhu2025sni} explored geometric and semantic representation, their semantic understanding remains in closed-volcabulary semantic segmentation and does not capture inter-object relations. We design an object grounding model that directly uses implicit features for spatial reasoning and show that leveraging the geometric features can improve the object grounding results.

\subsection{Submap-based neural SLAM}

An effective strategy for large-scale 3D reconstruction is to partition the scene into multiple submaps. \citet{kahler2016real} create submaps storing truncated SDF values based on visibility criteria and align them by optimizing the relative poses of overlapping keyframes.
MIPS-Fusion~\citep{tang2023mips} extends this idea by incrementally generating MLP-based submaps based on the camera’s field of view and aligning them via point-to-plane refinement. 
Vox-Fusion++~\citep{zhai2024vox} adopts a dynamic octree structure for each submap and performs joint camera tracking and submap alignment by optimizing a differentiable rendering loss. 
Loopy-SLAM~\citep{liso2024loopy} uses a neural-point-based approach, creating submaps upon large camera rotations and later constructing a pose graph with iterative closest point (ICP) to detect loop closures. More recently, PLGSLAM~\citep{deng2024plgslam} combines axis-aligned tri-planes for high-frequency features with an MLP for low-frequency components, enabling multiple local representations to be merged efficiently.  NEWTON~\citep{matsuki2024newton} employs a spherical coordinate system to create local maps that accommodate flexible boundary adjustments. 
Multiple-SLAM~\citep{liu2023efficient} and CP-SLAM \citep{hu2024cp} consider collaborative scenarios and fuse local neural implicit maps from multiple agents.
Although effective, existing methods require reconstructing the scene’s geometry to align submaps, which can be costly and inaccurate in real-world settings. In contrast, our method aligns submaps directly in the feature space via hierarchical optimization, providing both fast and robust performance without explicit geometric reconstruction. We also demonstrate that semantic features can help improve submap alignment in a complementary way to geometric information.

%% file: sections/overview.tex
\section{Overview}

We aim to design a method for spatial scene understanding that captures geometric and semantic information. We consider the following two tasks.

\myParagraph{Task 1: Signed distance field estimation} An SDF is a scalar function that assigns to each point in 3D space its distance to the nearest surface, together with a sign indicating which side of the surface the point lies on. Following common convention, the distance is positive outside of occupied space, negative inside, and zero on the boundary. Continuous and differential approximations of SDF are widely used in scene reconstruction and robot collision avoidance. Given RGB-D images, the robot builds an implicit map that allows SDF prediction for query points, which can be used for scene construction with marching cubes~\citep{lorensen1998marching}. To evaluate the quality of scene construction, we compute the Chamfer distance between the point clouds sampled on the constructed mesh and the ground-truth mesh.

\myParagraph{Task 2: Open-vocabulary object grounding} Open-vocabulary object grounding is the task of localizing an object in a 3D bounding box from a text description that may include a combination of attributes (e.g., ``red backpack'') and/or spatial relationships (e.g., ``the fridge near the brown table''). It is a key capability in scene understanding and is useful for many robotics tasks, \eg, instruction-based navigation and loco-manipulation. Given RGB-D images, the robot builds an implicit map that stores vision-language features, with which it can predict the 3D bounding box of the queried object in the text description. The quality of the object grounding is evaluated by computing the accuracy at certain IoU (Intersection over Union) thresholds between the predicted and ground-truth boxes.

In the following sections, we introduce our method, \AlgName, which achieves efficient geometric and semantic modeling and the ability to solve the two task above. We first introduce our hierarchical grid structure and feature mapping in \Cref{sec:local_mapping}. Each grid has its associated latent feature. The key to feature mapping is a query and decoding process that connects the latent features to the raw image features. To enable efficient feature mapping, we design an encoder to initialize the grid features so that the optimization can start from a good initial point, which is introduced in \Cref{sec:feature_initialization}. The obtained neural field can already be used for SDF construction. For object grounding, we introduce two ways of using our hierarchical 3D grids in \Cref{sec:open_vocabulary_understanding}: a training-free way using off-the-shelf models and feature queries, and our own model that directly takes in the 3D grids and uses SDF features to guide the bounding box prediction. Finally, we extend our method to SLAM by making sensor poses optimizable variables and introducing efficient submap alignment using the latent hierarchical grid features in \Cref{sec:SLAM_extension}.

%% file: sections/method.tex
\section{Methodology}


We introduce our method in this section and cover the details of each module mentioned in the overview.

\subsection{Hierarchical 3D grids}
\label{sec:local_mapping}

\subsubsection{Overview}

We first introduce our hierarchical 3D grids and the query and decoding process.

\begin{figure}[!ht]
    \centering
    \includegraphics[width=\linewidth]{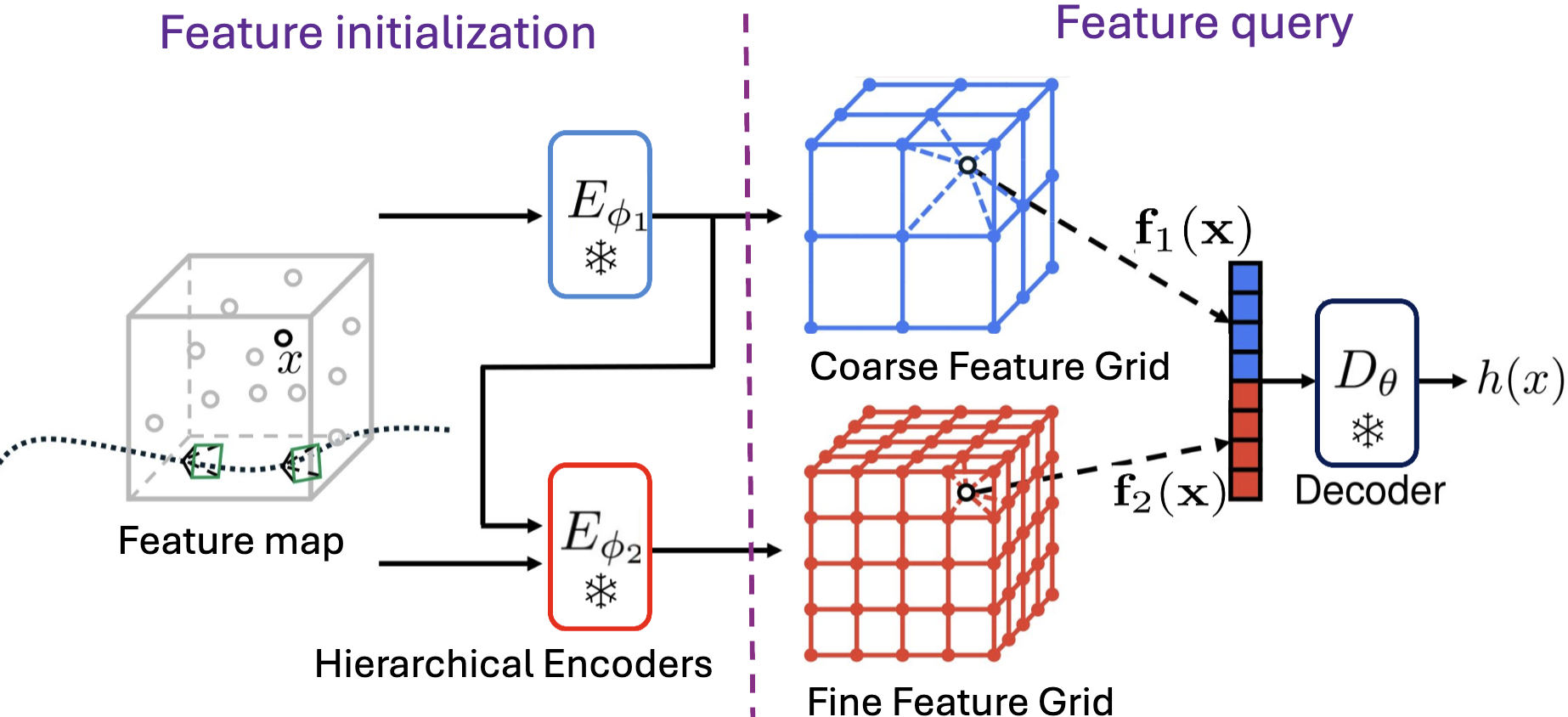}
    \caption{\textbf{Illustration of the hierarchical grids.} 
    Left part: hierarchical feature initialization using pretrained encder (\Cref{sec:feature_initialization}). Right part: given point cloud observations, \AlgName performs local feature mapping with multiresolution feature grids via feature query and decoding (\Cref{sec:local_mapping}). 
    }
    \label{fig:hierarchical_grid}
\end{figure}

\begin{definition}[(Multiresolution Feature Grid)]
\label{def:grid}
A \textit{multiresolution feature grid} contains $L>1$ levels of regular grids with increasing spatial resolution ordered from coarse ($l=1$) to fine ($l=L$).
Each level $l$ stores features $F_l = \{ \bff_{l,i} \}_{i=1}^{N_l} \in \bbR^{N_l \times d_l}$, where $d_l$ is the feature dimention at level $l$. For each vertex $\bfz_{l,i}$, there is a function $q_l(i)$ to index its corresponding feature $\bff_{l,q_l(i)}$. 
Together with a kernel function $k_l: \Real^3 \times \Real^3 \to \Real$, the feature grid defines a continuous feature field 
\begin{equation}
    f_l(\bfx)= \sum_{i \in I_l} k_l(\bfx, \bfz_{l,i}) \bff_{l,q_l(i)},
    \label{eq:neural_field_def}
\end{equation}
where $\bfx \in \Real^3$ is any query position, and $I_l$ indexes over involved vertices at level $l$.
To obtain the output $h(\bfx)$ (\textit{e.g.}, signed distance or vision-language feature) at query position $\bfx$, the features at different levels are concatenated (denoted by $\oplus$) and processed by a decoder network $D_\theta$,
\begin{equation} \label{eq:multiresolution_grid}
h(\bfx; F, \theta) = D_\theta \bigl(
	\oplus_{l \in [L]}
	f_l(\bfx)
\bigr).
\end{equation}
The model has the set of features from all levels $F$ and the decoder parameters $\theta$ as learnable parameters. An illustration of the feature query and decoding can be found in the right part of \Cref{fig:hierarchical_grid}. Details of grid size and latent feature dimensions can be found in \Cref{sec:grid_szie_and_feat_dim}.
\end{definition}

We implement the kernel functions $k_l$ using trilinear interpolation. Given the grid size, a scene is divided into $L_l \times W_l \times H_l$ grid indices with $L_l$, $W_l$, and $H_l$ corresponding to the number of indices along axes $x$, $y$, and $z$ respectively.
We demonstrate two separate types of features in the hierarchical grids.
Depending on the types of features, we introduce both dense and sparse grid features and the corresponding feature indexing function $q_l(i)$. 
For tasks where the grid feature dimension is low and free-space estimation is useful, \textit{e.g.}, SDF, we store dense grid features where there is a corresponding latent feature for every grid, \ie, the number of features $N_l = L_l \times W_l \times H_l$. In that case, the index function is simply $q_l(i)=i_x \times L_l + i_y \times W_l + i_z \times H_l$. For tasks where the grid feature dimension is high and free space is meaningless, \eg, vision-language features, we use sparse grids where we only store features for the near-surface grids. In that case, given the grid vertex $\bfz_{l,i}$, the index function $q_l(i)$ needs to either return an invalid index if there is no associated feature or the index in the feature array $F_l$. In the case of an invalid index, the associated feature $\bff_{l,q_l(i)}=\mathbf{0}$ so it does not play a role in \eqref{eq:neural_field_def}. We implement $q_l(i)$ as a lookup tensor $Q_l \in \bbR^{L_l \times W_l \times H_l}$, where value $-1$ indicates no associated feature, and values $1, ..., N_l$ indicate the index in $F_l$. The function then becomes $q_l(i)=Q_l(i_x, i_y, i_z)$. We choose this design since the 1-dimensional index tensor consumes negligible memory compared to the high-dimensional features, and it allows $\calO(1)$ lookup. One can choose different implementations based on the requirements, \eg, a hash function.

While the multiresolution feature grid offers a powerful representation, training the decoder $D_\theta$ in \eqref{eq:multiresolution_grid} online not only presents a computational burden but also leads to unreliable generalization or catastrophic forgetting~\citep{tosi2024nerfs}. 
To avoid this, we pre-train the decoder $D_\theta$ offline over multiple scenes, which are implemented with MLPs (details in \Cref{sec:map_decoders}). During online mapping, the decoder weights are fixed, and the robot only needs to optimize the grid features. 

Consider feature mapping in map $s$, we are given a set of keyframes with associated poses $\{\hat{T}_k^s\}_k$, and observations $\{X^k\}_k \ \{Y^k\}_k$, where each $X^k = \{\bfx^k_1,\ldots,\bfx^k_{m_k}\} \subset \Real^3$ is a point cloud observed in frame $k$, and $Y^k=\{\bfy^k_1,\ldots,\bfy^k_{m_k}\} \in \bbR^{d}$ are observations associated with each point, \eg, SDF values or vision language features. 
We provide an illustration of point feature lifting in Figure~\ref{fig:feature_lifting}. Given an RGB image, we first run SAM~\citep{kirillov2023segment} to get object masks. Each object mask is then cropped into an image patch, and we run CLIP~\citep{clip} to get features for each cropped image patch. The features are lifted to 3D by using the depth image, producing the point cloud $\{X^k\}_k$ and associated features $\{Y^k\}_k$. For SDF measurements, we only need depth images or LiDAR scans. Following \citet{ortiz2022isdf}, we sample points along the rays. For near-surface points, the SDF values are computed as the distance from the points to the surface along the ray direction; For free-space points, we compute lower and upper bounds, which we will describe in more detail in the following paragraphs.

\begin{figure}[t]
    \centering
    \includegraphics[width=\linewidth]{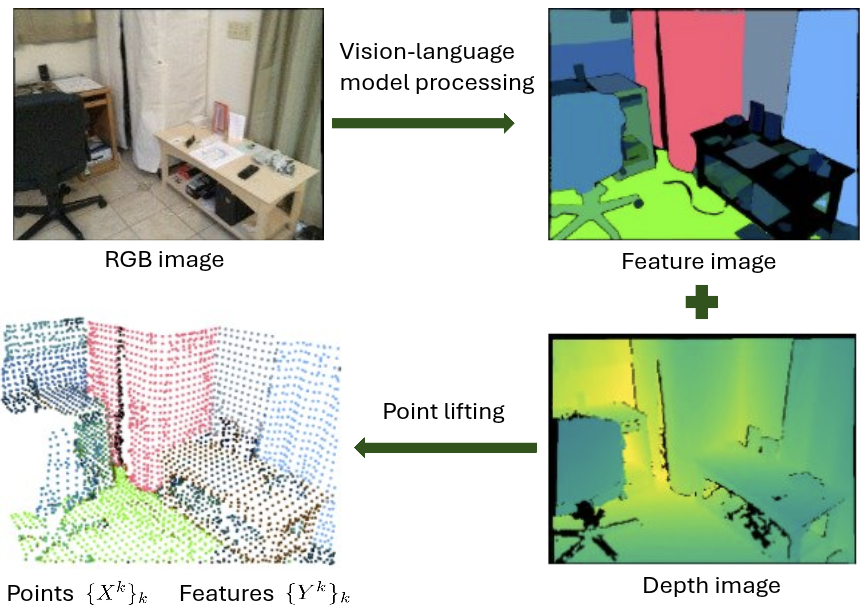}
    \caption{\textbf{Illustration of feature lifting.} Given an RGB image, we obtain image-plane features from a vision-language model (visualized via principal component analysis). The pixel coordinates can be transformed to 3D using a corresponding depth image. The lifted 3D points and associated features are used to query and optimize the latent 3D grids in \AlgName.}
    \label{fig:feature_lifting}
\end{figure}

\begin{problem}[(Local feature mapping)] \label{prob:local_mapping}
Given $n$ images with associated poses $\{T_k^s\}_k$ in the reference frame of submap $s$ and associated point-cloud observations $\{X^k\}_k$, the local mapping problem is defined as,
\begin{equation}
\label{eq:local_mapping}
\underset{F}{\min}
	 \!
	 \sum_{k=1}^n \sum_{j=1}^{m_k} 
        c_j\bigl(
            h(T^s_k \bfx^k_j; F)
        \bigr), 
\end{equation}
where $c_j$ is a cost function associated with the $j$-th observation. We drop the dependence of $h$ on the decoder parameters $\theta$ to reflect that the decoder is trained offline and fixed during mapping. We introduce how we define cost functions for different features in the following.
\end{problem}

\subsubsection{Cost functions}
\label{sec:cost_functions}
Our hierarchical grid model can be applied to different types of features, depending on the choice of cost functions $c_j$. We introduce the cost functions for the SDF and vision-language fields separately below.

\myParagraph{Cost functions for SDF reconstruction}
We follow iSDF \citep{ortiz2022isdf} to design cost terms for SDF reconstruction from point cloud observations. Specifically, we classify all observed points as either (i) on or near surface (default to $30$ cm as in iSDF), or (ii) in free space. For on or near surface observations, the cost function $c_j = c_j^{\text{sdf}}$ is based on direct SDF supervision,
\begin{equation}
	c_j^{\text{sdf}} \left(h^{\text{sdf}}(\bfx_j)\right) = w_j^\text{sdf} \left |h^{\text{sdf}}(\bfx_j) - y^{\text{sdf}}_j \right|,
	\label{eq:sdf_residual}
\end{equation}
where $w_j^\text{sdf}>0$ is a weight (default to $5.4$) and $y^{\text{sdf}}_j \in \Real$ is a measured SDF value on or near surface obtained in the same way as iSDF by computing distance along the ray from the points to the surface intersection point.

For free-space observations, we use the cost to enforce bounds on the SDF values. Specifically, we follow iSDF to obtain lower and upper bounds $\underline{b}_j, \bar{b}_j$ on the SDF from sensor measurements, and define $c_j = c_j^{\text{bnd}}$ as,
\begin{align}
    c_j^\text{lo} \left(h^{\text{sdf}}(\bfx_j) \right) &= \max (e^{\beta (\underline{b}_j - h^{\text{sdf}}(\bfx_j))} - 1, \; 0), \\
    c_j^\text{up} \left(h^{\text{sdf}}(\bfx_j) \right) &= \max (h^{\text{sdf}}(\bfx_j) - \bar{b}_j, \; 0), \\
    c_j^{\text{bnd}} \left(h^{\text{sdf}}(\bfx_j) \right) &= \max(c_j^\text{lo} \left(h^{\text{sdf}}(\bfx_j) \right), c_j^\text{up} \left(h(\bfx_j) \right) ).
    \label{eq:bnd_residual}
\end{align}
This cost applies an exponential penalty ($\beta=5$ by default) for the lower bound and a linear penalty for the upper bound.
This is because, in practice, violation of the lower bound is usually more critical, \textit{e.g.}, if $\underline{b}_j = 0$ and the model predicts negative SDF values. We do not include Eikonal regularization \citep{gropp2020implicit} because we observed that it has limited impact on accuracy while making the optimization slower.

\myParagraph{Cost functions for vision-language reconstruction}
For vision-language features, points in the free space have no semantic meaning, so we only consider on-surface points. Given the feature $\bfy^{vl}_j$ obtained from a vision-language model (we choose CLIP~\citep{clip}) with corresponding 3D coordinate $x_j$ obtained from a depth image or LiDAR scan), we compute the cosine similarity loss
\begin{align}
    c^{\text{vl}}_j\left(h^{\text{vl}}(\bfx_j)\right) = 1 - \cos\left( h^{\text{vl}}(\bfx_j), \bfy^{\text{vl}}_j \right).
\end{align}
\subsubsection{Grid sizes and latent feature dimensions}
\label{sec:grid_szie_and_feat_dim}
By default, we have two levels of features, and we set the grid size to be $0.5$ cm at the coarse level and $0.1$ cm at the fine level. For SDF features, the latent grid feature dimension at each level is set to $4$, while for vision-language features, the dimension at each level is set to $64$.


\subsection{Hierarchical feature initialization}
\label{sec:feature_initialization}

In practice, the bulk of the computational cost is incurred by the optimization over the high-dimensional grid features $F$.
To address this challenge, we propose a method that leverages the structure of the hierarchical grids to learn to initialize $F$ from sensor observations. 
Our key intuition is that, at any level, an effective initialization can be obtained by accounting for optimization results from the previous levels.

In the following, we use $F_l$ to denote the subset of latent features at level $l$,
and $F_{1:l}$ denote all latent features up to and including level $l$.
We consider the problem of initializing $F_l$ given fixed submap poses and coarser features $F_{1:l-1}$. 
This amounts to solving the following subproblem of \Cref{prob:local_mapping},
\begin{equation}
\label{eq:level_local_mapping}
\underset{{F_l}}{\min}
	\quad \sum_{k=1}^n \sum_{j=1}^{m_k} 
		c_j \left(
		h(T^s_k \bfx^k_j; F_{1:l-1}, F_l, 0_{l+1:L})
		\right),
\end{equation}
where we explicitly expand $F$ into the (known) coarser features $F_{1:l-1}$, 
the target feature to be initialized $F_l$, and the finer features (assumed to be zero).

To develop our approach, we first present theoretical analysis and derive a closed-form solution to \eqref{eq:level_local_mapping} in the linear-least-squares case. Using insight from the analytical solution, we develop a learning approach to initialize the grid features at each level applicable to the general (nonlinear) case in \eqref{eq:level_local_mapping}.

\myParagraph{Special case: linear least squares}
Suppose that the decoder $D_\theta$ in \Cref{def:grid} is a linear function, and assume that the cost function $c_j$ in \eqref{eq:level_local_mapping} is quadratic, \eg, $c_j(h(\bfx_j)) =  (h(\bfx_j) - \bfy_j)^2$. Problem \eqref{eq:level_local_mapping} becomes a linear least squares problem, for which we can obtain a closed-form solution from the normal equations, as shown next.

\begin{proposition}[]
\label{lem:linear_case}
With linear decoder $D_\theta$ and quadratic costs $c_j(h(\bfx_j)) =  (h(\bfx_j) - \bfy_j)^2$,
the optimal solution to \eqref{eq:level_local_mapping} is:
\begin{equation}
    F_l^\star 
    = E(r_{1:l-1}(\bfx))
    := - \left[
    J^\top J
    \right]\pinv
    J^\top r_{1:l-1}(\bfx),
    \label{eq:level_local_mapping_closed_form}
\end{equation}
where $\bfx = \{T^s_k \bfx^k_j\}$ and $\bfy = \{\bfy_j\}$ collect all observed points and labels in two vectors,
$J = \partial h(\bfx;F) / \partial F_l$ is the Jacobian matrix evaluated at $x$, 
and $r_{1:l-1}(\bfx)$ are the residuals of prior levels, represented in vector form as,
\begin{equation}
	r_{1:l-1}(\bfx) = h(\bfx; F_{1:l-1}, 0_{l:L}) - \bfy.
    \label{eq:level_l_residuals}
\end{equation}
\end{proposition}
\begin{proof}
Please refer to Appendix~\ref{sec:proof}.
\end{proof}

Proposition \ref{lem:linear_case} reveals an interesting structure of the optimal initialization $F_l^\star$:
it is as a function of the prior levels' residuals $r_{1:l-1}(x)$. We will build on this insight to approach the problem in the general case.

\myParagraph{General case: learning hierarchical initialization}
We take inspiration from \Cref{lem:linear_case} to develop a learning method for feature initialization in the general case. 
We replace $E$ in \Cref{lem:linear_case} with a neural network \textit{encoder} $E_{\phi_l}$ to approximate $F^\star_l$ from the residuals $r_{1:l-1}(\bfx)$,
\begin{equation}\label{eq:encoder}
    F_l^\star \approx E_{\phi_l} (r_{1:l-1}(\bfx)),
\end{equation}
where $\phi_l$ are the encoder parameters. 
Before predicting the coarsest-level features, its corresponding features are initialized randomly. After one level of features are initialized, we compute the residuals for the next level following \eqref{eq:level_l_residuals}, and use the corresponding encoders \eqref{eq:encoder} to initialize the remaining levels in a recursive way. An illustration can be found in the left part of \Cref{fig:hierarchical_grid}. For different types of features, we train a separate encoder $E_{\phi_l}$ to initialize the feature grid $F_l$ at each level $l$. 
The input to the encoder is represented as a point cloud. Each 3D position $\bfx_j \in \Real^3$ in the submap frame has an associated feature residual. The initial point residual features are first pooled onto the 3D voxel grid that they lie in. See \Cref{fig:encoder} for an illustration. For the SDF features, we then use a 3D CNN to process the pooled grid features. For the vision-language features, since we only store features in surface cells, we use a Fast Point Transformer (FPT) \citep{park2022fast} to process the pooled features. Finally, we use an MLP to predict the grid features. Our encoder design is illustrated in \Cref{fig:encoder}.  Next, we present the residual inputs for both the SDF and the visual-language features.

\begin{figure}[t]
    \centering
    \includegraphics[width=\linewidth]{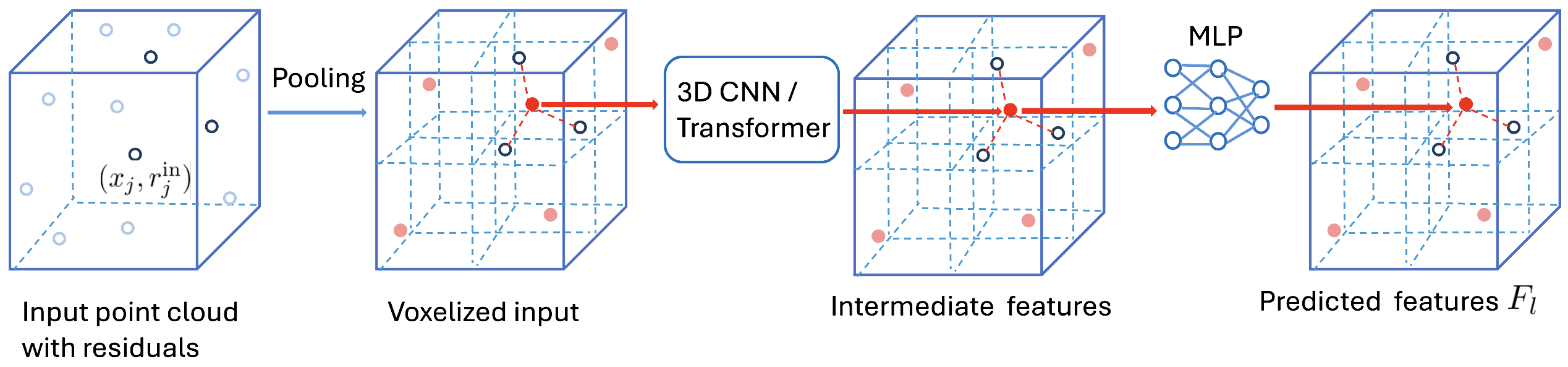}
    \caption{Illustration of the level-$l$ encoder $E_{\phi_l}$. Input point cloud with residuals $\{x_j, r_j^\text{in}\}_j$ is voxelized via averaging pooling and processed by a 3D CNN or Fast Point Transformer~\citep{park2022fast} depending on whether the grid is dense or sparse. The neural network outputs at all vertices are then transformed via an MLP to predict the target feature grid $F_l$. }
    \label{fig:encoder}    
\end{figure}

\myParagraph{SDF input residuals}
For SDF features, we use the following measurement residuals to construct the encoder input
$r_j^{\text{sdf}} = \begin{bmatrix}
r_{j,1}^{\text{sdf}} &  r_{j,2}^{\text{sdf}}  & r_{j,3}^{\text{sdf}} 
\end{bmatrix}^\top \in \Real^3$:
\begin{align}
	r_{j,1}^{\text{sdf}} &= \begin{cases}
	h^{\text{sdf}}(\bfx_j) - y_j, & \text{if $\bfx_j$ near surface,} \\
	0, & \text{otherwise,}
	\end{cases} 
	\label{eq:encoder_input_sdf} \\ 
	r_{j,2}^{\text{sdf}} &= \begin{cases}
	\max (h(\bfx_j) - \bar{b}_j, 0), & \text{if $\bfx_j$ in free space,} \\
	0, & \text{otherwise,}
	\end{cases}
	\label{eq:encoder_upper_bnd} \\
	r_{j,3}^{\text{sdf}} &= \begin{cases}
	\max (\underline{b}_j - h(\bfx_j), 0), & \text{if $\bfx_j$ in free space,} \\
	0, & \text{otherwise.}
	\label{eq:encoder_lower_bnd}
	\end{cases}
\end{align}
The first residual $r^{\text{in}}_{j,1}$ corresponds to the SDF cost \eqref{eq:sdf_residual}.
The remaining two features correspond to the residuals of the upper and lower bounds used to compute \eqref{eq:bnd_residual}.

\myParagraph{Visual-language residuals}
We only have one type of residual for the vision-language features:
\begin{align}
    r_{j}^{\text{vl}} = h^{\text{vl}}(\bfx_j) - y^{\text{vl}}_j
\end{align}
Similar to the decoder, we train the encoders offline using submaps from multiple environments.
When training the level $l$ encoder $E_{\phi_l}$,
we use a training loss based on \eqref{eq:level_local_mapping},
\begin{equation}
\label{eq:encoder_training_loss}
\underset{{\phi_l}}{\min} 
	\sum_{s \in S} \sum_{j \in J_s} 
		c_j \left(
		h(\bfx^s_j; F^s_{1:l-1}, E_{\phi_{l}}(r^s_{1:l-1}), 0_{l+1:L})
		\right),
\end{equation}
where $S$ contains the indices of all training submaps, $J_s$ contains the indices of all points in submap $s$, and $F^s$ denotes the features for submap $s$.
During training, we use noisy poses within the submaps to compute $\bfx^s_j$ to account for the possible pose estimation errors at test time. Combined with the initialization encoder, our local mapping approach is summarized in \Cref{alg:hier_local_mapping}. Details of encoder implementation and training can be found in \Cref{sec:map_encoders}.


\begin{algorithm}[t]
	\caption{\small \textsc{Hierarchical Local Mapping}}
	\label{alg:hier_local_mapping}
	\begin{algorithmic}[1]
		\small 
		\Function{$\{T_k^s\}_k, F$ = HierarchicalLocalMap}{}
        \For{level $l = 1, 2, \hdots, L$}
            \State Initialize features at level $l$:
             $F_{l} \leftarrow E_{\phi_{l}}(r_{1:l-1}(x)). $
		\EndFor \label{alg:hier_local_mapping:init}
		\State From the initialized values, update features $F$ by minimizing \eqref{eq:local_mapping}.
            \State \Return $F$.
		\EndFunction
	\end{algorithmic}
\end{algorithm}

\subsection{Open-vocabulary object grounding}
\label{sec:open_vocabulary_understanding}

\begin{figure*}[t]
    \centering
    \begin{subfigure}[t]{0.5\linewidth}
        \centering
        \includegraphics[width=\linewidth]{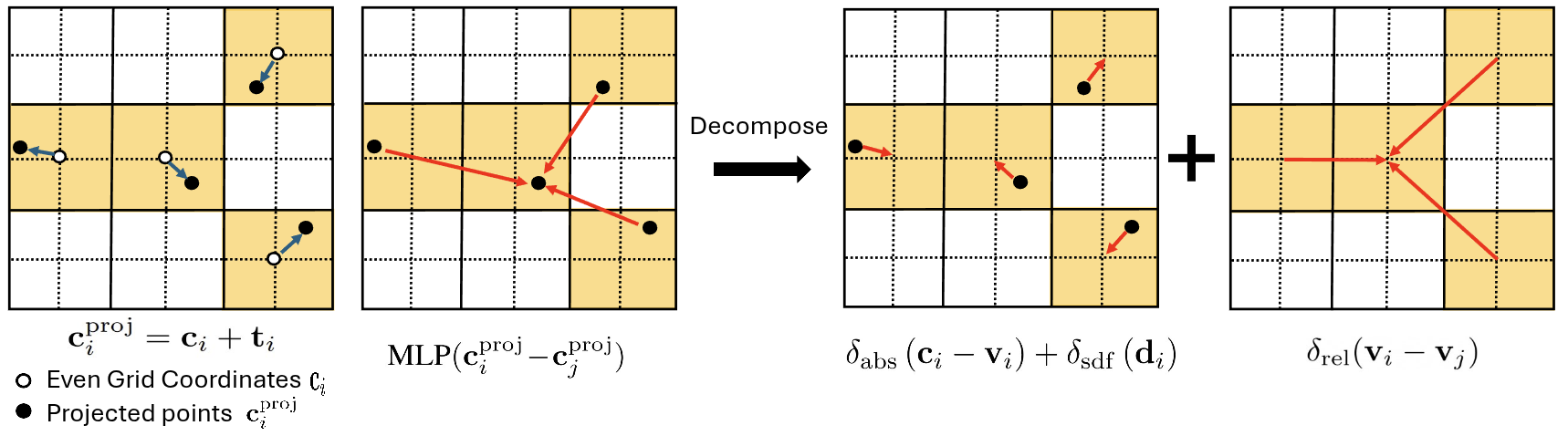}
        \caption{Our model: SDF-guided encoding.}
        \label{fig:sdf_guided_encoding}
    \end{subfigure}
    \hspace{0.5cm}%
    \begin{subfigure}[t]{0.45\linewidth}
        \centering
        \includegraphics[width=\linewidth]{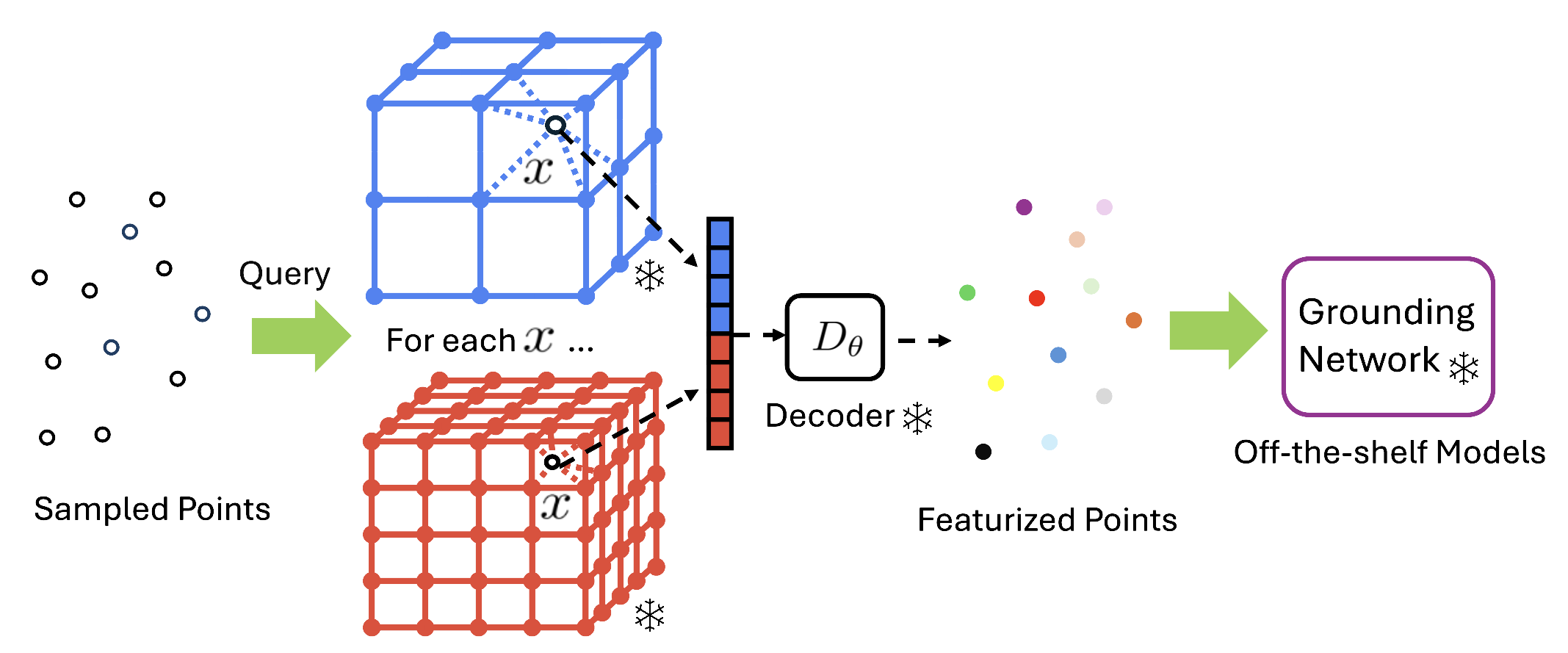}
        \caption{Training-free inference.}
        \label{fig:training_free_inference}
    \end{subfigure}%
    \caption{\textbf{\AlgName for open-volcabulary object grounding.}
    (a) We design an object grounding model using SDF features to get fine-grained structure information from regular grids. Since the grid coordinates are distributed evenly in the space, they do not contain enough spatial information. We assume the SDF feature indicates the surface projection of these coordinates ({\textcolor[RGB]{0,0,168} {blue arrows}}) so we can compute attention using the projected coordinates $\bfc_i^{\text{proj}}$ ({\textcolor[RGB]{240,0,0} {red arrows}}). We follow FPT~\citep{park2022fast} to use coordinate decomposition and integrate SDF features for the continuous positional embedding.
    (b) Given sampled points, we query and decode point features from \AlgName, which can be fed to an off-the-shelf model for object grounding.
    }
    \label{fig:object_grounding}
\end{figure*}

We introduced hierarchical grids as an efficient neural field representation for capturing geometric (SDF) and semantic (vision-language) structure of the environment.
In this section, we show that the vision-language field is useful for spatial semantic understanding by demonstrating it on open-vocabulary object grounding.
We first design our object grounding model that leverages both geometric and semantic grid features to get accurate box prediction for the queried object. Then, we show that our feature maps can be used by existing models for training-free inference.

\subsubsection{SDF-guided object grounding} We design our object grounding model
\label{sec:sdf_guided_grounding}
that takes in grid features, query text, and predicts the 3D bounding box of the queried object. 
Since the query text may involve position information (\eg, one object next to another object),
capturing spatial information and feature correlation in neural network models is needed, which usually requires some sort of neighborhood search. Compared to an unordered point cloud that needs time-consuming search, \eg with K Nearest Neighbors (KNN) \citep{qi2017pointnet++,zhao2021point}, 3D grids allow $\calO(1)$ neighborhood queries. 
Despite the query efficiency, relying directly on the evenly distributed grid coordinates to predict object bounding boxes leads to inaccurate positions or box sizes since the grid coordinates are not necessarily on the object surface and fail to capture fine-grained structure. Instead, we supplement the vision-language features with the SDF features, which implicitly specify the distance to the object surfaces. 
We use a Fast Point Transformer (FPT) \citep{park2022fast} to apply attention to the grid features. We first give a brief overview of FPT and then describe how we integrate SDF features to guide the FPT attention. Given input point cloud $\mathcal{P}^{\mathrm{in}} = \{(\mathbf{p}_n, \mathbf{i}_n)\}$ with coordinates $\bfp_n$ and features $\bfi_n$, FPT first voxelizes $\mathcal{P}^{\mathrm{in}}$ into voxels $\mathcal{V} = \{(\mathbf{v}_i, \mathbf{f}_i, \bfc_i)\}$, where $\bfv_i$ are the voxel indices, $\bff_i$ are the aggregated voxel features, and $\bfc_i$ are the voxel centers, i.e., average point coordinates in the voxel. 
%
For each voxel center $\bfc_i$, FPT performs self-attention within local neighbor indices $\mathcal{N}(i)$, formulated as:
\begin{equation}
    \mathbf{f}'_i
    =
    \sum_{j \in \mathcal{N}(i)}
    a\!\left(
        \mathbf{f}_i,
        \delta(\bfc_i,\bfc_j) 
    \right) \psi(\bff_j),
\label{eq:fpt_raw_attn}
\end{equation}
where $\mathbf{f}'_i$ is the output feature, $a\!\left(\mathbf{f}_i,\delta(\bfc_i,\bfc_j)\right)$ is the attention weight between voxel $i$ and $j$, computed based on feature $\bff_i$ and positional encoding $\delta(\bfc_i, \bfc_j)$, and $\psi$ is the value projection layer. 

\myParagraph{Space complexity reduction} To reduce space complexity, they decompose the relative centroid positions as:
\begin{equation}
    \bfc_i - \bfc_j = \left(\bfc_i - \bfv_i\right) - \left(\bfc_j - \bfv_j\right) + \left(\bfv_i - \bfv_j\right).
\label{eq:fpt_raw_decomposition}
\end{equation}
%
With this decomposition, the model proposes different positional embeddings $\delta_{abs}(\bfc_i - \bfc_i)$ for continuous coordinates and $\delta_{rel}(\bfv_i - \bfv_j)$ for discrete coordinates. It then computes the centroid-aware voxel feature 
\begin{equation}
    \bfg_i = \bff_i + \delta_{\mathrm{abs}} \left( \bfc_i-\bfv_i \right),
    \label{eq:fpt_original_feature}
\end{equation}
and reformulates the attention layer in \eqref{eq:fpt_raw_attn} as
\begin{equation}
    \mathbf{f}'_i
    =
    \sum_{j \in \mathcal{N}(i)}
    a\!\left(
        \mathbf{g}_i,
        \delta_{\mathrm{rel}}
        \left(
            \mathbf{v}_i-\mathbf{v}_j
        \right)
    \right)
    \psi(\mathbf{g}_j).
    \label{eq:fpt_decomp_attn}
\end{equation}
%
The different positional embeddings $\delta$, $\delta_{abs}$ and $\delta_{rel}$ are implemented with MLPs. \eqref{eq:fpt_decomp_attn} is an approximate of the initial attention formulation \eqref{eq:fpt_raw_attn}. Please refer to \citet[Sec. 3.3]{park2022fast} for why this reduces space complexity.

\myParagraph{SDF-guided attention} With the preliminary introduced above, we now design SDF-guided attention. Starting with the center positional embedding before coordinate decomposition, \ie, $\delta(\bfc_i, \bfc_j)$ in \eqref{eq:fpt_raw_attn}, since the grid coordinates distribute evenly in the space, the term $\bfc_i - \bfc_j$ will have uniform intervals, thus losing fine-grained structure information. We mitigate this by using the SDF features with the intuition that the SDF features can indicate the surface projection point of a grid coordinate. 
Denoting the vector that projects a grid coordinate $\bfc_i$ to the nearest surface by $\bft_i$, we can compute the point on the surface as:
\begin{equation}
    \bfc_i^{\mathrm{proj}} = \bfc_i + \bft_i. 
\end{equation}
Replacing $\bfc$ with $\bfc^{\text{proj}}$ in \eqref{eq:fpt_raw_decomposition}, we get
\begin{equation}
    \bfc_i^{\mathrm{proj}} \!-\! \bfc_j^{\mathrm{proj}} \!=\! ( \bfc_i \!-\! \bfv_i + \bft_i) \!-\! (\bfc_j \!-\! \bfv_j  + \bft_j) \!+\! \left(\bfv_i \!-\! \bfv_j\right)\!.
    \notag
\end{equation}
Equation \eqref{eq:fpt_original_feature} then becomes
\begin{equation}
    \bfg_i = \bff_i + \delta_{\mathrm{abs}} \left( \bfc_i-\bfv_i + \bft_i \right). 
    \label{eq:fpt_new_feature}
\end{equation}
However, the projection vector $\bft_i$ is unknown. Given that the SDF features provide information about distance, we approximate $\bft_i \approx \text{MLP}_{1}(\bfd_i)$ with the corresponding SDF feature $\bfd_i$ of grid $\bfc_i$ using an MLP. Thus, \eqref{eq:fpt_new_feature} can be approximated as
%
\begin{align}
    \bfg_i &= \bff_i + \delta_{\mathrm{abs}} \left( \bfc_i-\bfv_i + \bft_i \right) \notag \\ 
    &\approx \bff_i + \delta_{\mathrm{abs}} \left( \bfc_i-\bfv_i \right) + \text{MLP}_2 (\bft_i) \notag \\
    &\approx \bff_i + \delta_{\mathrm{abs}} \left( \bfc_i-\bfv_i \right) + \text{MLP}_2 \left(\text{MLP}_1(\bfd_i) \right) \notag \\ 
    &= \bff_i + \delta_{\mathrm{abs}} \left( \bfc_i-\bfv_i \right) + \delta_{\mathrm{sdf}} \left(\bfd_i \right),
\end{align}
where $\delta_{\mathrm{sdf}} (\cdot)$ is the combination of $\text{MLP}_{2} \left(\text{MLP}_{1}(\cdot) \right)$. See Figure~\ref{fig:sdf_guided_encoding} for an illustration of SDF-guided encoding. For the decoder, we follow the structure in Locate-3D~\citep{mcvay2025locate} and add a similar SDF embedding to \eqref{eq:fpt_new_feature} in the positional encoding (pe) of the point coordinates,
\begin{equation}
    \text{pe}_i = \text{MLP}(\bfc_i) + \text{MLP}(\bfd_i),
\end{equation}
where $\bfd_i$ is the associated SDF features. We use the fine-level grid as the input to our neural network.

We use the grid size and latent feature dimension setting introduced in \Cref{sec:grid_szie_and_feat_dim}. Following \citet{mcvay2025locate}, we map DINO features~\citep{oquab2024dinov2} in addition to CLIP~\citep{clip}. To train our neural network, we use the same loss functions as \citet{mcvay2025locate}. Our model outputs $M=256$ object proposals. For each proposal, we predict the confidence of target word positions in the sentence in the form of one-hot vector, object bounding box in the form of $(x_{\text{min}}, y_{\text{min}}, z_{\text{min}}, x_{\text{max}}, y_{\text{max}}, z_{\text{max}})$, and object grid coordinate mask, leading to output tensors $P_{\text{target}} \in \bbR^{M\times L_{\text{max}}}$, $P_{\text{box}}=\bbR^{M \times 6}$, and $P_{\text{mask}}=\bbR^{M\times G}$, where $L_{\text{max}}$ is the max length of the query text and $G$ is the number of cells. We set the mask value to $1$ for a cell if it contains ground-truth object points. To compute the loss, the predicted objects and ground-truth objects are first matched using the Hungarian algorithm \citep{kuhn1955hungarian} with distance composed of the product of predicted target world confidence and ground-truth words referring to the target object (binary masks) plus the L1 norm and generalized IoU (Intersection over Union) \citep{rezatofighi2019generalized} between predicted objects and ground-truth objects. Then, we compute the following loss functions: 1) focal loss \citep{lin2017focal} for target word prediction; 2) $L_1$ loss and generalized IoU loss for object bounding-box prediction; 3) cross-entropy and dice loss \citep{cheng2021per} for grid mask prediction. More details of our object grounding model training can be found in \Cref{sec:object_grounding_model_details}.

\subsubsection{Training-free inference} 
The feature query~\eqref{eq:neural_field_def} and decoding~\eqref{eq:multiresolution_grid} enable our grid map to be used by existing object grounding models that take in featurized point clouds, as long as we store the features needed by the object grounding model. For example, Locate-3D~\citep{mcvay2025locate} takes in a point cloud with per-point CLIP~\citep{clip} and DINO~\citep{oquab2024dinov2} features, and predicts the 3D object bounding box given a query text. The query points can be obtained from depth images / LiDAR scans. In that case, our grid map serves as a memory-efficient representation. Once we map a scene, the map can be used by existing object grounding models in a training-free manner. See \Cref{fig:training_free_inference} for an illustration. 
We conduct experiments with Locate-3D~\citep{mcvay2025locate} and demonstrate that we can achieve substantial memory compression with small accuracy loss in Section~\ref{sec:scene_understanding_exp}.

\subsection{Extension to SLAM}
\label{sec:SLAM_extension}

In this section, we extend our feature mapping approach to a SLAM method by introducing camera pose optimization and hierarchical alignment of multiple submaps.

\subsubsection{Camera pose optimization}
\label{sec:pose_optimizaion}
In the SLAM setting, we rewrite \eqref{eq:local_mapping} by treating the camera poses as optimization variables and introducing a pose regularization term:
\begin{equation}
\label{eq:local_slam}
    \underset{F, \{T_k^s\}_k \subset \SE(3)}{\min}
	 \!
	 \sum_{k=1}^n \! \sum_{j=1}^{m_k} \!
        c_j\bigl(
            h(T^s_k \bfx^k_j; F)
        \bigr) \!\! + \!\!
         \sum_{k=1}^n \! \rho(\That^s_k, T^s_k),
\end{equation}
where $c_j$ is a cost function as \eqref{eq:local_mapping} and $\rho: \SE(3) \times \SE(3) \to \Real $ is the pose regularization term. 
We use regularization inspired by trust-region methods \citep[Ch.~4]{nocedal1999numerical},
\begin{equation}
\rho(\That, T) = w^\rho \max \bigl( \bigl \| \Log(\That^{-1}T) \bigr \|_2 - \tau, 0 \bigr ),
\label{eq:trust_region}
\end{equation}
which penalizes pose updates larger than the trust-region radius $\tau$, and $w^\rho$ is a weight parameter.

The formulation in \eqref{eq:local_slam} can be performed in an incremental manner.
At each time step $k$, the received observation, \eg, depth image and 2D vision-language features, introduces new cost terms $\{c_j\}_{j=1}^{m_k}$ in \eqref{eq:local_slam}.
We then alternate between tracking and mapping to update the estimated robot pose and the neural implicit map.
During tracking, we only optimize the current robot pose $T^s_k$ and keep the submap features $F$ fixed.
We estimate the robot motion by minimizing \eqref{eq:local_slam} with respect to $T^s_k$.
During mapping, we fix all pose estimates $T_{1:k}^s$ and only optimize the submap features $F$.
In the incremental setting, the encoder initialization is disabled. Instead, we use the latest frame and randomly sampled historical historical frames to update the submap features $F$ by minimizing \eqref{eq:local_slam}.
In our implementation, we approximate each pose variable locally as $T^s_k =  \That^s_k \Exp(\varepsilon^s_k)$ where $\varepsilon^s_k \in \Real^6$ is a local pose correction on the Lie algebra of SE(3).
The grid features $F$ and the correction terms $\{\varepsilon^s_k\}_k $ are jointly optimized using Adam \citep{kingma2014adam} in PyTorch \citep{paszke2019pytorch}.

\subsubsection{Global submap alignment and fusion}
\label{sec:global_optimization}

As the robot navigates in a large environment or for an extended time, its onboard pose estimation will inevitably drift.
To achieve globally consistent mapping, it is necessary to accurately align and fuse the local submaps with respect to the global frame.
State-of-the-art systems, such as MIPS-Fusion \citep{tang2023mips} and Vox-Fusion++ \citep{zhai2024vox}, align submaps using their learned SDF values. However, this is computationally expensive and susceptible to noise.
In this section, we develop a hierarchical method for submap alignment and fusion, which attains significant speed-up and accuracy improvements by performing optimization directly using the features from the hierarchical submap grids.

\begin{figure}[t]
    \centering
    \includegraphics[width=0.9\linewidth]{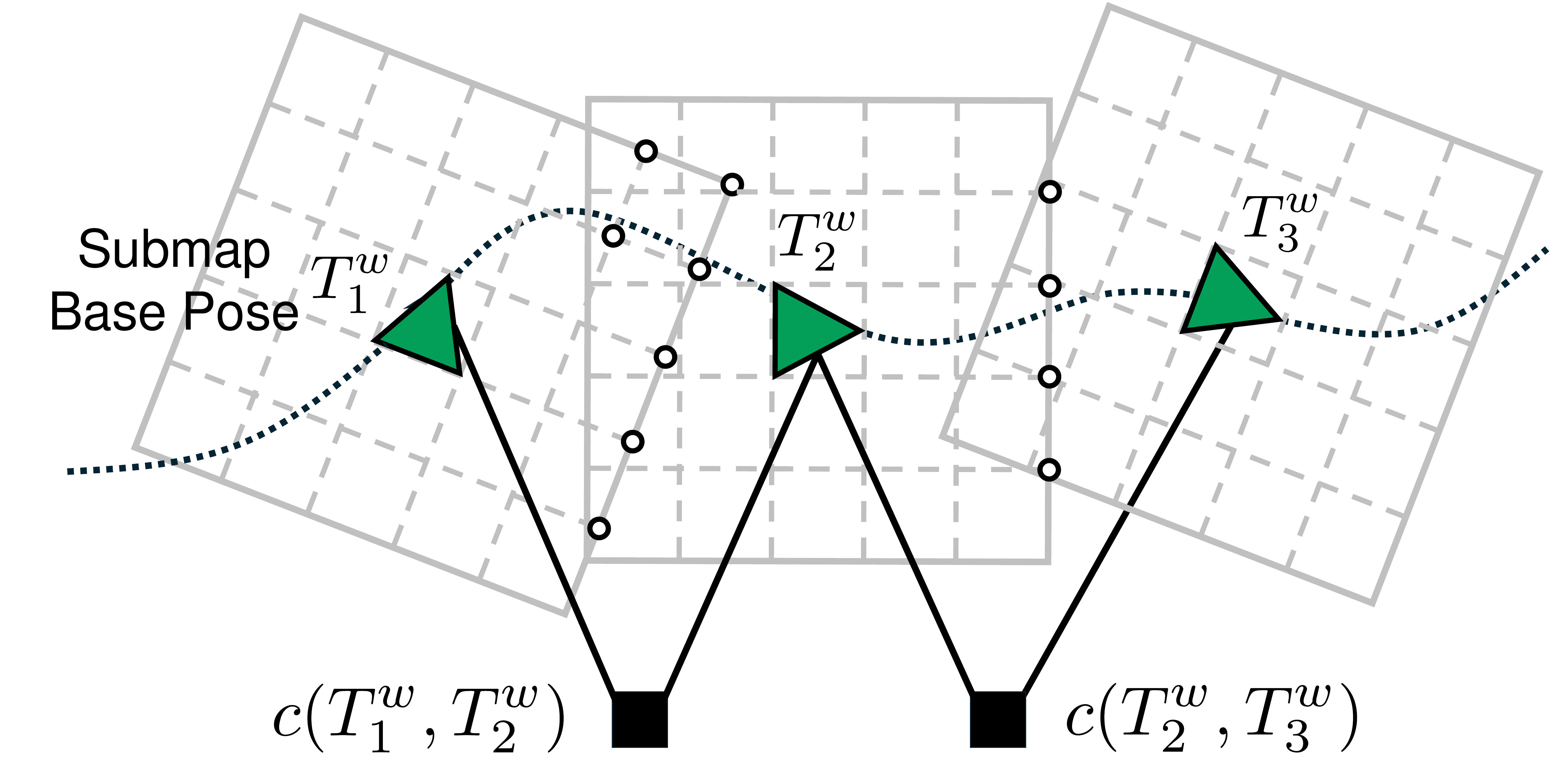}
    \caption{\textbf{Illustration of hierarchical submap alignment.} Given locally optimized submaps, \AlgName performs global alignment and fusion across submaps
    to eliminate estimation drift and achieve globally consistent neural fields (Section \ref{sec:global_optimization}). The cost function $c(T^{w}_i, T^{w}_j)$ can be either the distance between latent features \eqref{eq:latent_pgo} or the distance between predictions \eqref{eq:pred_pgo}.}
    \label{fig:hierarchical_align}
\end{figure}

\myParagraph{Hierarchical Submap Alignment}
Consider the problem of aligning a collection of $n_s$ submaps, each represented as a hierarchical feature grid from \Cref{sec:local_mapping}.
For each submap $u \in [n_s]$, we aim to optimize the submap base pose in the world frame, denoted as $T^w_u \in \SE(3)$.
The key intuition for our approach is that, for any pair of submaps to be well-aligned, their implicit feature fields should also be aligned.
We formulate a \textit{hierarchical} and \textit{correspondence-free} approach to exploit this intuition.

Our method performs alignment by progressively including features at finer levels.
Let $\bff^u_l(\bfx) \in \Real^d$ denote the level-$l$ interpolated feature at query position $\bfx$ in submap $u$.
Furthermore, let $f^u_{1:l}(\bfx)$ denote the result after concatenating features up to and including level $l$, \ie, $f^u_{1:l}(\bfx) = \bigoplus_{l'=1}^{l} f^u_{l'}(\bfx) \in \Real^{ld}$. Consider a pair of overlapping submaps $u,v \in [n_s]$.
Let $\{\bfz^u_{l,i}\} \subset \Real^3$ denote the vertex positions at level $l$ in submap $u$.
Using these vertices, we define the following pairwise cost to align features,
\begin{equation}
\cfeat_l(T^w_u\!, T^w_v) \!\! = \!\!\!\!\! \sum_{i \in I^{uv}_l} \!\!\!
d \bigl(
	f^u_{1:l}(\bfz^u_{l,i}), f^v_{1:l}((T^w_v)^{\!-1} \!(T^w_u) \bfz^u_{l,i})
\bigr).
\label{eq:latent_pairwise_cost}
\end{equation}
In \eqref{eq:latent_pairwise_cost}, $I^{uv}_l$ denotes the indices of level-$l$ vertices in submap $u$ that lie within the overlapping region of the two submaps. See \Cref{fig:hierarchical_align} for an illustration.
Intuitively, the right-hand side of \eqref{eq:latent_pairwise_cost} compares feature vectors interpolated from the two submaps.
The first feature comes from the source grid $u$ evaluated at its grid vertex position $\bfz^u_{l,i}$.
To evaluate the corresponding feature in the target grid $v$, we use the submap base poses to transform the vertex position, \ie, $\bfz^v_{l,i} = (T^w_v)^{-1} (T^w_u) \bfz^u_{l,i}$ before querying the target feature grid.
Finally, $d$ denotes a distance metric in the space of implicit features.
In our implementation, we use the L2 distance, $d(f, f') = \norm{f - f'}^2_2$, for SDF features and cosine similarity distance, $d(f, f') = 1 - 
\text{cos}(f, f')$ for vision-language features.
Given the pairwise alignment costs defined in \eqref{eq:latent_pairwise_cost}, 
\AlgName performs joint submap alignment by formulating and solving a problem similar to pose graph optimization.
Let $\Ecal$ denote the set of submap pairs with overlapping regions.
Then, we jointly optimize all submap poses $\{T^w_u\}_u \subset \SE(3)$
as follows.

\begin{problem}[Level-$l$ submap alignment]
\label{prob:feature_alignment}
Given $n_s$ submaps with current base pose estimates $\{\That^w_u\}_u$, 
solve for updated submap base poses via,
\begin{equation}
\underset{\{T^w_u\}_u \subset \SE(3)}{\min}
	\sum_{(u,v) \in \Ecal} \cfeat_l(T^w_u, T^w_v)
	+ 
	\sum_{u=1}^{n_s} \rho(\That^w_u, T^w_u),
\label{eq:latent_pgo}
\end{equation}
where $\rho$ is the trust-region regularization defined in \eqref{eq:trust_region}.
\end{problem}

\begin{algorithm}[t]
\caption{\small \textsc{Hierarchical Submap Alignment}}
\label{alg:hier_alignment}
\begin{algorithmic}[1]
    \small 
    \Function{$\{T^w_u\}_u$ = SubmapAlignment}{}
    \State Initialize submap poses $\{T^w_u\}_u$.
    \For{level $l = 1, 2, \hdots, L$} \label{alg:hier_align:feat_start}
        \State Update $\{T^w_u\}_u$ by solving \eqref{eq:latent_pgo}
        at level $l$ for $k_{f,l}$ iters.
    \EndFor \label{alg:hier_align:feat_end}
    \State Update $\{T^w_u\}_u$ by solving \eqref{eq:pred_pgo} for $k_s$ iters. \label{alg:hier_align:sdf}
    \State \Return $\{T^w_u\}_u$.
    \EndFunction
\end{algorithmic}
\end{algorithm}

Similar to local SLAM,  we solve \eqref{eq:latent_pgo} using PyTorch \citep{paszke2019pytorch} where the poses are updated by optimizing local corrections (represented in exponential coordinates) to the initial pose estimates $\{\That^w_u\}_u$. 
Our formulation leads to a sequence of alignment problems that include features at increasingly fine levels.
We solve these problems sequentially, using solutions from level $l$ as the initialization for level $l+1$; see lines~\ref{alg:hier_align:feat_start}-\ref{alg:hier_align:feat_end} in \Cref{alg:hier_alignment}.

To further enhance accuracy, the submap pose estimates may be finetuned during a final alignment stage using predicted SDF values or vision-language features. Since only a few iterations are needed in typical scenarios, this approach preserves computation efficiency compared to other methods that directly use SDF for alignment. We define the following pairwise alignment cost for submap pair $(u,v)$,
\begin{equation}
	c^{\text{pred}}(T^w_u\!, T^w_v) \! = \!\!\!\!
	\sum_{j \in J^{uv}} \!\! d\bigl(
	h^u (\bfx^u_j; F^u), h^v( (T^w_v)^{-1} T^w_u \bfx^u_j; F^v)
	\bigr)\!.
	\label{eq:sdf_pairwise_cost}
\end{equation}
In \eqref{eq:sdf_pairwise_cost}, $J^{uv}$ contains the indices of observed points that are in the intersection region of the two submaps.
For each observation $j$, $\bfx^u_j \in \Real^3$ denotes its position in the frame of submap $u$.
Compared to \eqref{eq:latent_pairwise_cost}, in \eqref{eq:sdf_pairwise_cost} we minimize the difference of the final predictions from both submaps. 
Same as above, $d(f, f') = \norm{f - f'}^2_2$ for SDF features and $d(f, f') = 1 - \text{cos}(f, f')$ for vision-language features.
Using this in the pose-graph formulation leads to a prediction-based submap alignment.
\begin{equation}
\underset{\{T^w_u\}_u \subset \SE(3)}{\min}
	  \sum_{(u,v) \in \Ecal} c^{\text{pred}}(T^w_u, T^w_v)
	+ 
	\sum_{u=1}^{n_s} \rho(\That^w_u, T^w_u),
\label{eq:pred_pgo}
\end{equation}
which follows the same format as \eqref{eq:latent_pgo} but changes the distance between latent features to the distance between predictions.
See Algorithm \ref{alg:hier_alignment} for the summary of submap alignment.
In Section \ref{sec:experiments:ablations}, we demonstrate that the combination of feature-based and prediction-based submap alignment yields the best performance in both robustness and computational efficiency.

\myParagraph{Submap Fusion}
So far, we addressed the problem of aligning submaps in the global frame to reduce estimation drift. In some applications, there is an additional need to extract a global representation (\eg, a consistent neural field or mesh) of the entire environment from the collection of local submaps. 
In \AlgName, we achieve this by using the average feature from all submaps to decode the global scene.
For any submap $u$, let $f^u(\bfx^u)$ denote the output of its multiresolution feature field evaluated at a position $\bfx^u \in \Real^3$ in the submap frame.
Given any query coordinate in the world frame $\bfx^w \in \Real^3$, we first compute the weighted average of all submap features,
\begin{equation}
f^w\!(\bfx^w) \!\!=\!\! 
\bigl (\sum_{u=1}^{n_s}\! w_u \!(\bfx^w) \bigr)^{\!-1}\!
\sum_{u=1}^{n_s} \!w_u \!(\bfx^w) f^u( (T^w_u)^{\!-1} \bfx^w),
\end{equation}
where each submap is associated with a binary weight $w_u (\bfx^w)$ computed using its bounding box,
\begin{equation*}
w_u (\bfx^w) = \begin{cases}
1 & \text{if $\bfx^w$ is inside submap $u$},\\
0 & \text{otherwise.}
\end{cases}
\end{equation*}
The final prediction is obtained by passing the average feature to the decoder network,
\begin{equation}
h^w(\bfx^w) = D_\theta (f^w(\bfx^w)).
\label{eq:fused_prediction}
\end{equation}
In summary, the proposed scheme achieves submap fusion via an averaging operation in the implicit feature space and it applies to different types of features.

%% file: sections/experiments.tex
\section{Evaluation}

In this section, we evaluate our method for scene construction, scene understanding, and pose estimation. For each part of the evaluation, we first introduce the experiment setup, including the dataset, baselines, and metrics. Then, we show both quantitative and qualitative results to demonstrate the effectiveness of our method.



\input{sections/exp_scene_construction}
\input{sections/exp_scene_understanding}
\input{sections/exp_pose_estimation}

\input{sections/exp_outdoor}

\input{sections/exp_ablation_studies}

%% file: sections/exp_scene_construction.tex
\subsection{Scene construction}

\begin{table*}[t]
\centering
\caption{Evaluation of local mapping quality for different methods on ScanNet \citep{dai2017scannet}.
\AlgName is optimized for 20 epochs and the baselines \iSDF and \Point are optimized for 100 epochs.
For each scene, we report optimization time (sec), Chamfer-L1 error (cm), and F-score (\%) computed using a threshold of $5$ cm. The best and second-best results are highlighted in \textbf{bold} and \underline{underline}, respectively.}
\label{tab:scannet_mae}
\resizebox{\textwidth}{!}{%
\begin{tabular}{|c|ccc|ccc|ccc|ccc|}
\hline
Scene
& \multicolumn{3}{c|}{0000}
& \multicolumn{3}{c|}{0011}
& \multicolumn{3}{c|}{0024}
& \multicolumn{3}{c|}{0207} \\

Method
& Time$\downarrow$ & C-l1$\downarrow$ & F-score$\uparrow$
& Time$\downarrow$ & C-l1$\downarrow$ & F-score$\uparrow$
& Time$\downarrow$ & C-l1$\downarrow$ & F-score$\uparrow$
& Time$\downarrow$ & C-l1$\downarrow$ & F-score$\uparrow$ \\
\hline

iSDF \citep{ortiz2022isdf}
& 67.71 & 4.77 & 77.87
& 25.89 & \underline{6.35} & \underline{67.32}
& 39.80 & \textbf{4.98} & \underline{73.94}
& 26.17 & \underline{7.32} & \underline{59.65} \\

Neural Points \citep{pan2024pin}
& \underline{57.87} & 12.23 & 40.89
& \underline{22.41} & 9.62 & 49.87
& \underline{32.84} & 10.63 & 46.34
& \underline{21.29} & 10.20 & 44.26 \\


\AlgName
& \textbf{1.58} & \textbf{4.04} & \textbf{83.32}
& \textbf{0.82} & \textbf{5.48} & \textbf{73.76}
& \textbf{0.97} & \underline{5.03} & \textbf{76.59}
& \textbf{0.74} & \textbf{5.91} & \textbf{70.45} \\


\hline
\end{tabular}%
}
\end{table*}

We first focus on the geometric part, \ie, SDF estimation, and evaluate \AlgName on the task of scene construction.

\myParagraph{Experiment setup} We use ScanNet \citep{dai2017scannet} as the main dataset, which features a collection of real-world RGB-D sequences with accurate camera poses and 3D reconstructions. We construct the SDF given depth image sequences and extract the scene mesh from the estimated SDF. To evaluate the mesh construction error, we compute the Chamfer-L1 distance and F-score against the ground-truth mesh. For baseline methods, we compare \AlgName against the MLP-based representation from iSDF \citep{ortiz2022isdf} and the neural-point-based representation from PIN-SLAM \citep{pan2024pin}, using default parameters from their open-source code. In the following, we refer to these two baselines as \iSDF and \Point, respectively.

In this part of the experiments, we represent the entire scene as a single submap and use the default grid sizes and SDF feature dimension introduced in \Cref{sec:grid_szie_and_feat_dim}.
We run \AlgName and the baseline \iSDF and \Point methods using ground truth poses, where pose optimization is disabled. 
For each method, we report its GPU time in addition to Chamfer-L1 distance and F-score.
\AlgName is optimized for 20 epochs.
For the \iSDF and \Point baselines, since they do not have access to pre-training, we optimize both for 100 epochs for a fair comparison.

\myParagraph{Results} As shown in \cref{tab:scannet_mae}, \AlgName achieves either the best or second-best reconstruction results on all scenes.
Using ground truth poses, \AlgName achieves superior speed, requiring only $0.7$–$1.5$~sec for optimization.
We then show some qualitative results. \Cref{fig:scannet_mapping} shows a qualitative comparison of the estimated SDF at a fixed height on scene 0207. Figure~\ref{fig:scannet_mapping:init} shows the initialization produced by the learned encoder (corresponding to line~\ref{alg:hier_local_mapping:init} in \cref{alg:hier_local_mapping}), which already captures the scene geometry to a large extent.
The remaining errors and missing details are improved after running 20 optimization epochs, as shown in Figure~\ref{fig:scannet_mapping:opt}.
Because \iSDF uses a single MLP to represent the scene, its output (see Figure~\ref{fig:scannet_mapping:isdf}) is overly smooth, which leads to lower recall and F-score than \AlgName.
Lastly, the SDF prediction from \Point (\cref{fig:scannet_mapping:pin}) is especially noisy in free space, due to the lack of neural point features far away from the surface.

To provide more insight on the performance of the learned hierarchical encoders, 
Figure~\ref{fig:encoder_io} visualizes the inputs and output predictions on scene 0024.
At each level (shown as a row in the figure), we visualize two channels $r^\text{sdf}_1, r^\text{sdf}_3$ of the voxelized input residuals defined in \eqref{eq:encoder_input_sdf} and \eqref{eq:encoder_lower_bnd}, the predicted SDF, and the predicted 3D mesh.
The level 1 encoder, although with a coarse resolution of $0.5$~m, already captures the rough scene geometry and free space SDF. 
On top of this coarse prediction, the level 2 encoder is able to add fine details and recover objects such as the sofa and the table.   

\begin{figure}[!ht]
    \centering
    \begin{subfigure}[t]{0.25\linewidth}
        \centering
        \includegraphics[trim=10 0 10 0, clip, width=\linewidth]{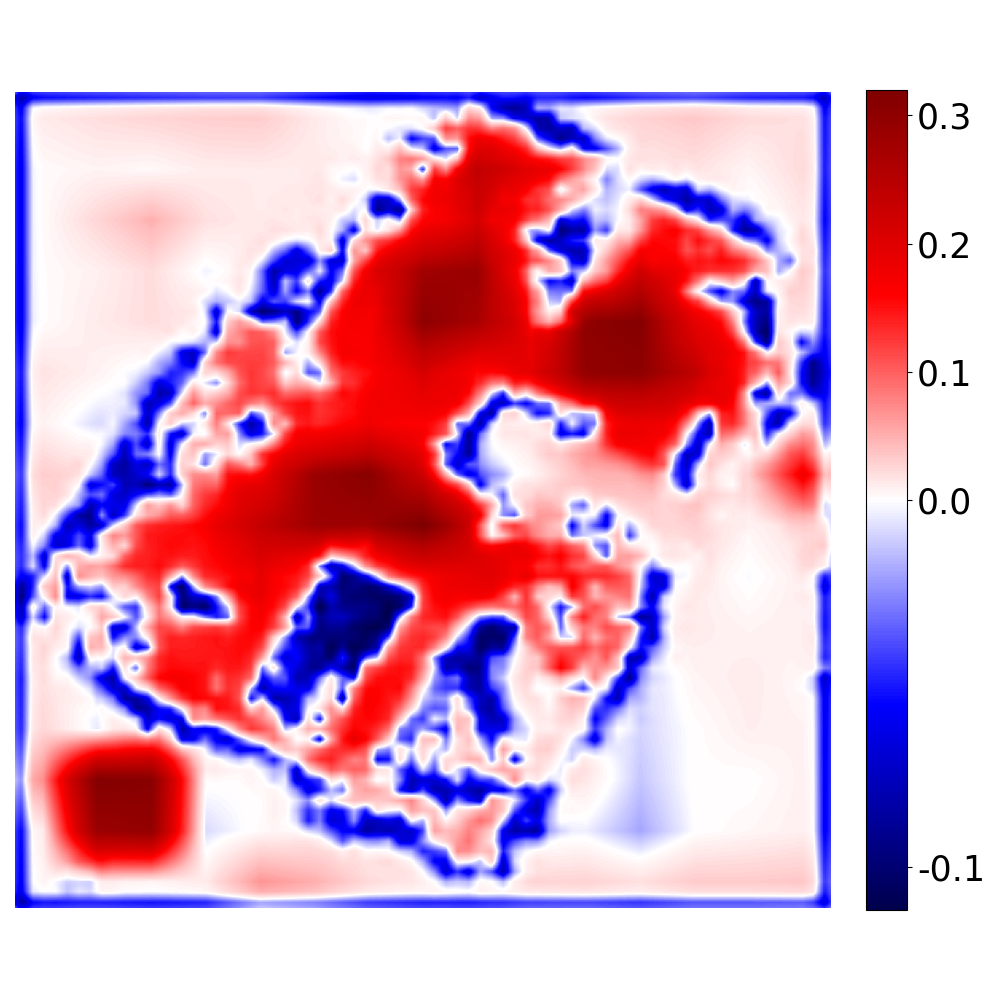}
        \caption{\AlgName (init)}
        \label{fig:scannet_mapping:init}
    \end{subfigure}%
    \begin{subfigure}[t]{0.25\linewidth}
        \centering
        \includegraphics[trim=10 0 10 0, clip, width=\linewidth]{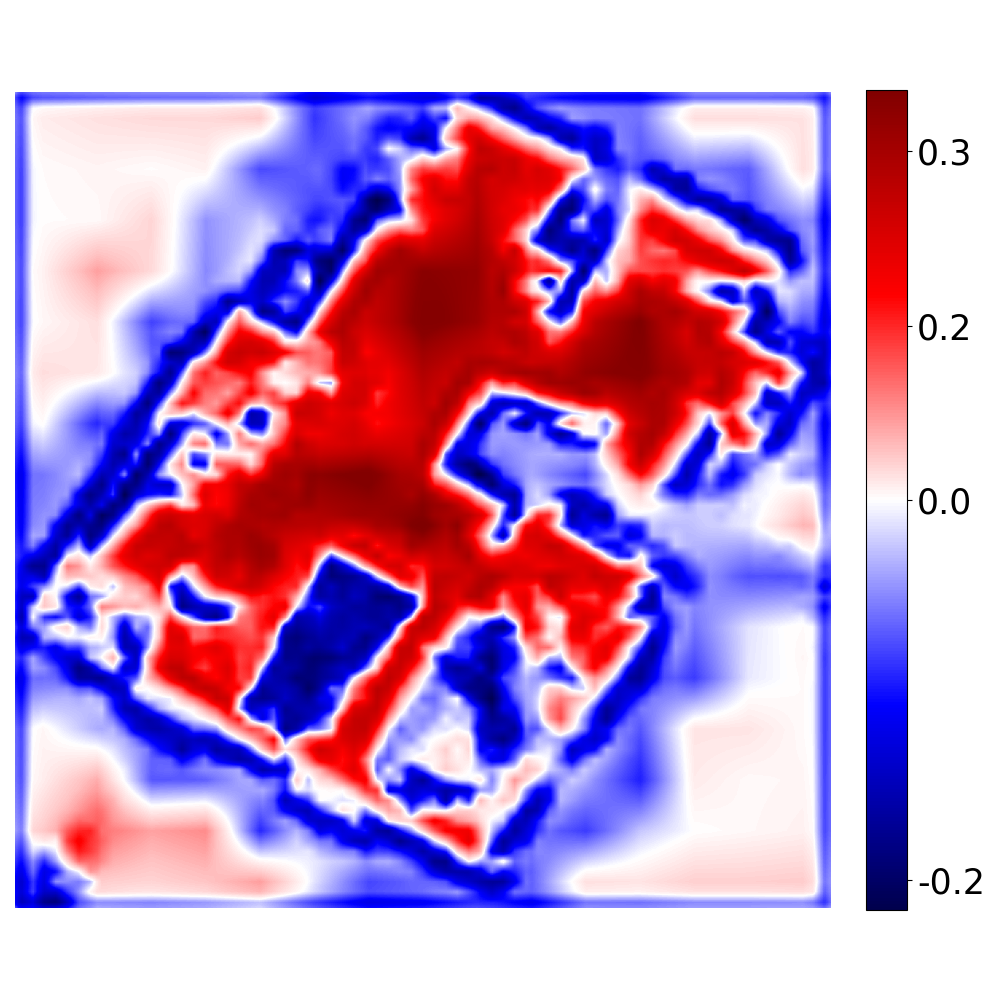}
        \caption{\AlgName (opt)}
        \label{fig:scannet_mapping:opt}
    \end{subfigure}%
    \begin{subfigure}[t]{0.25\linewidth}
        \centering
        \includegraphics[trim=10 0 10 0, clip, width=\linewidth]{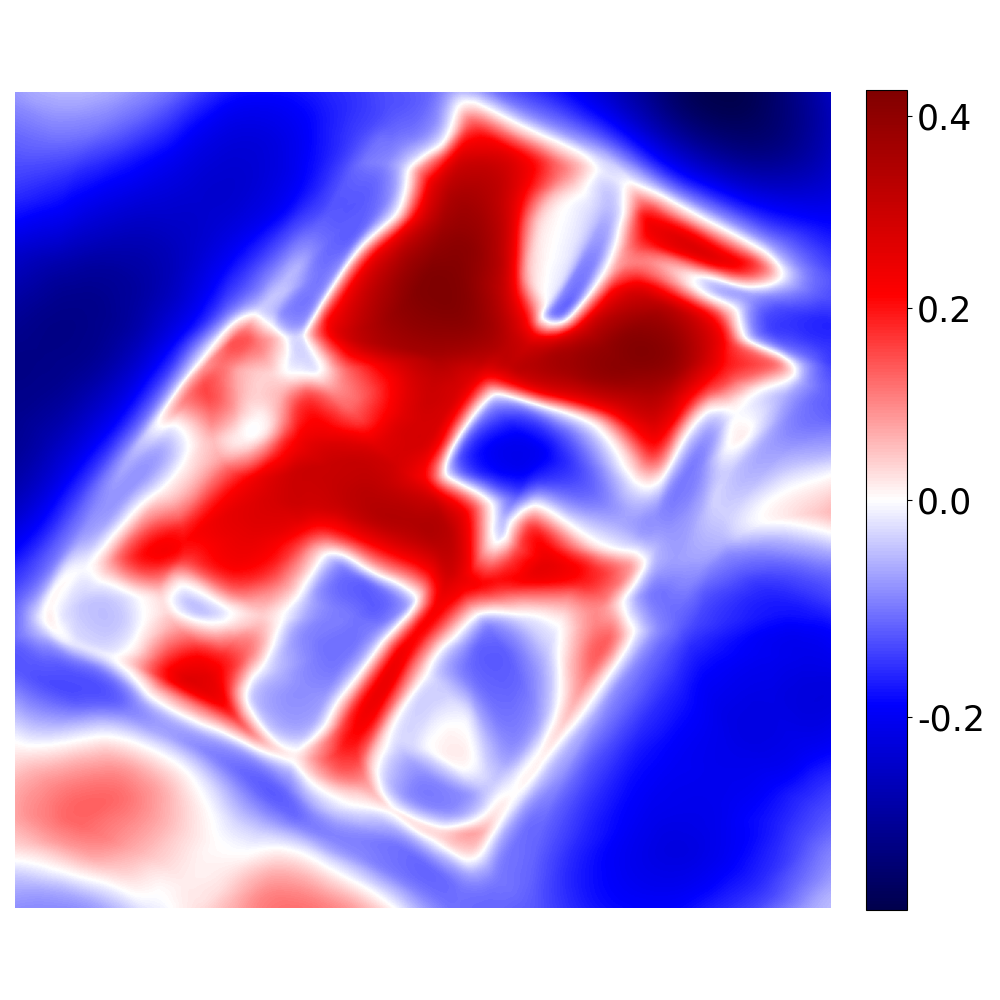}
        \caption{\iSDF}
        \label{fig:scannet_mapping:isdf}
    \end{subfigure}%
    \begin{subfigure}[t]{0.25\linewidth}
        \centering
        \includegraphics[trim=10 0 10 0, clip, width=\linewidth]{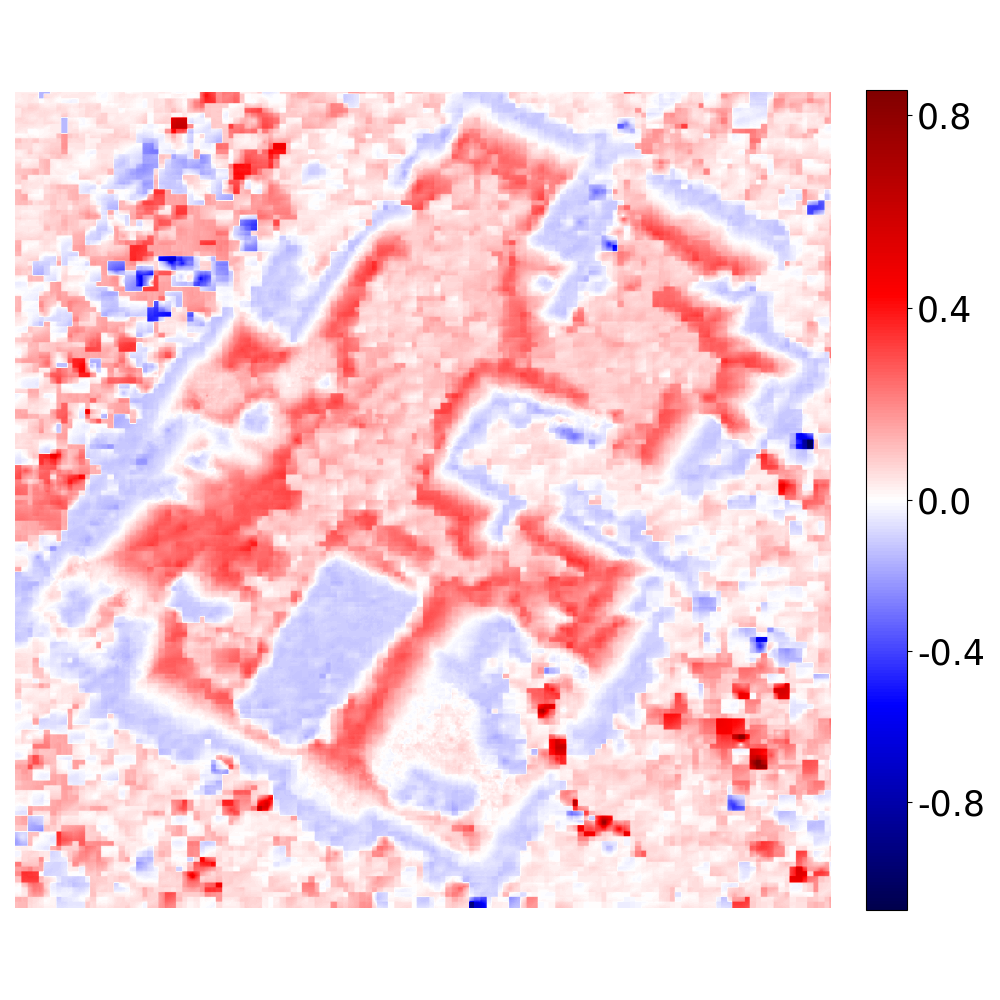}
        \caption{\Point}
        \label{fig:scannet_mapping:pin}
    \end{subfigure}
    \caption{Visualization of estimated SDF at a fixed height on {ScanNet} scene 0207. 
    }
    \label{fig:scannet_mapping}
\end{figure}

\begin{figure}[!ht]
    \centering
    \begin{subfigure}[t]{0.24\linewidth}
        \centering
        \includegraphics[trim=0 0 0 0, clip, width=\linewidth]{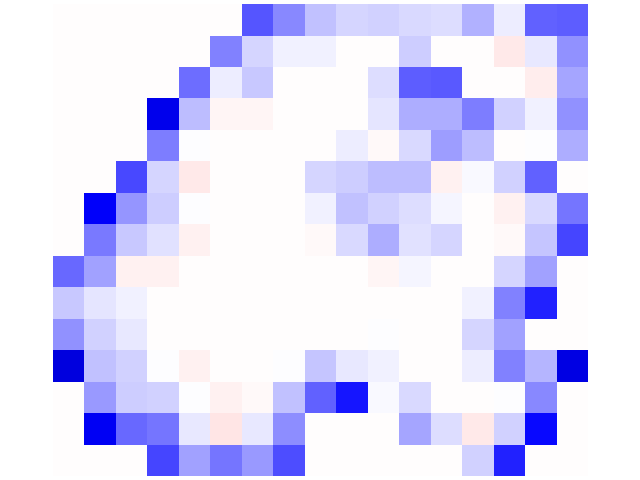}
        \caption{Level 1 $r^\text{in}_1$}
    \end{subfigure}
    \begin{subfigure}[t]{0.24\linewidth}
        \centering
        \includegraphics[trim=0 0 0 0, clip, width=\linewidth]{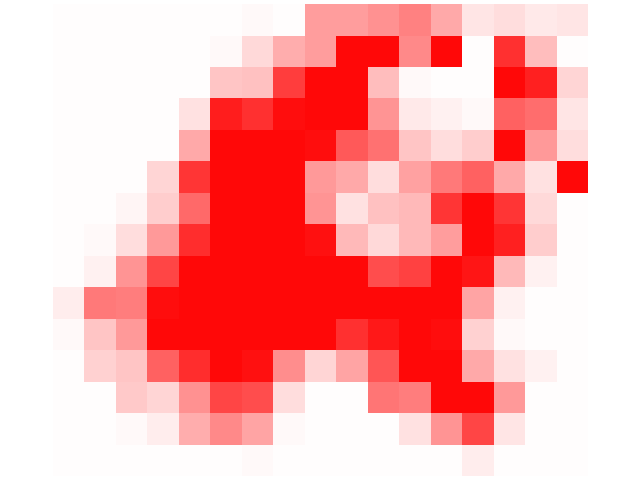}
        \caption{Level 1 $r^\text{in}_3$}
    \end{subfigure} 
    \begin{subfigure}[t]{0.22\linewidth}
        \centering
        \includegraphics[trim=0 0 0 0, clip, width=\linewidth]{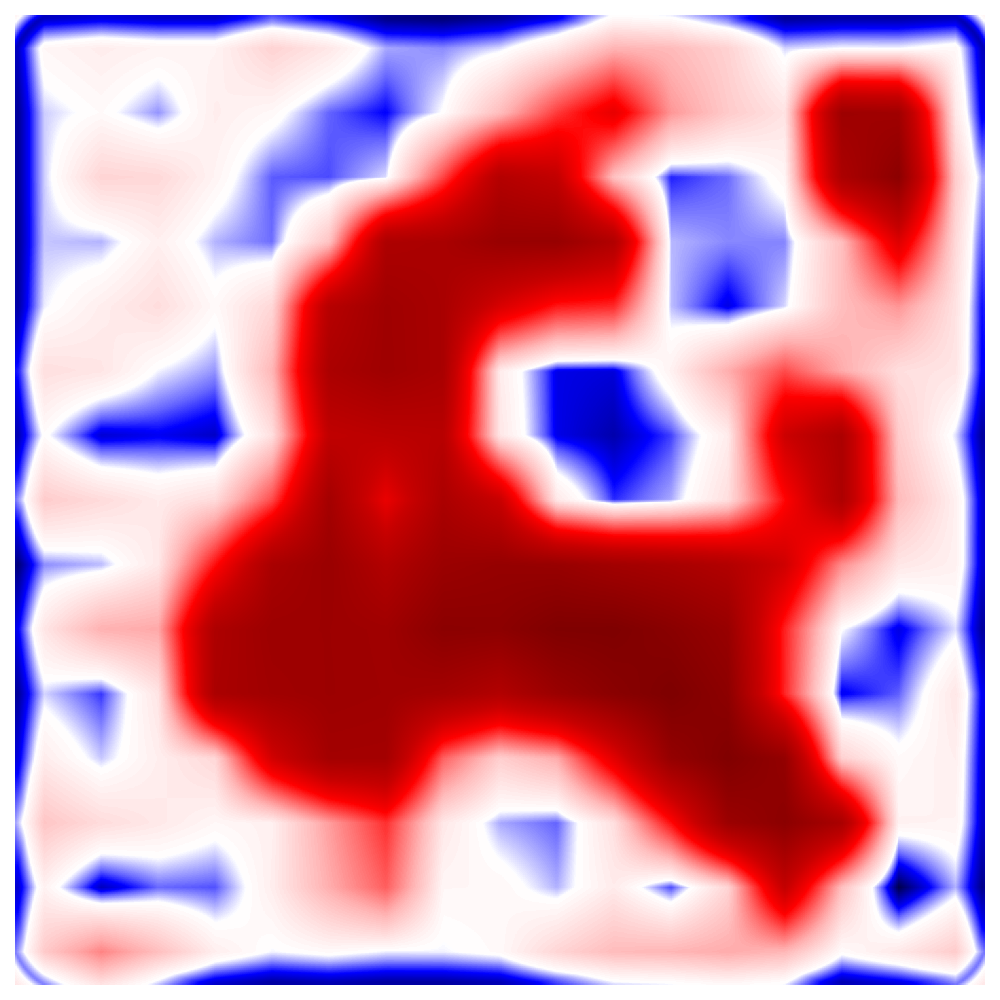}
        \caption{Level 1 SDF}
    \end{subfigure}
    \begin{subfigure}[t]{0.24\linewidth}
        \centering
        \includegraphics[trim=400 0 400 0, clip, width=\linewidth]{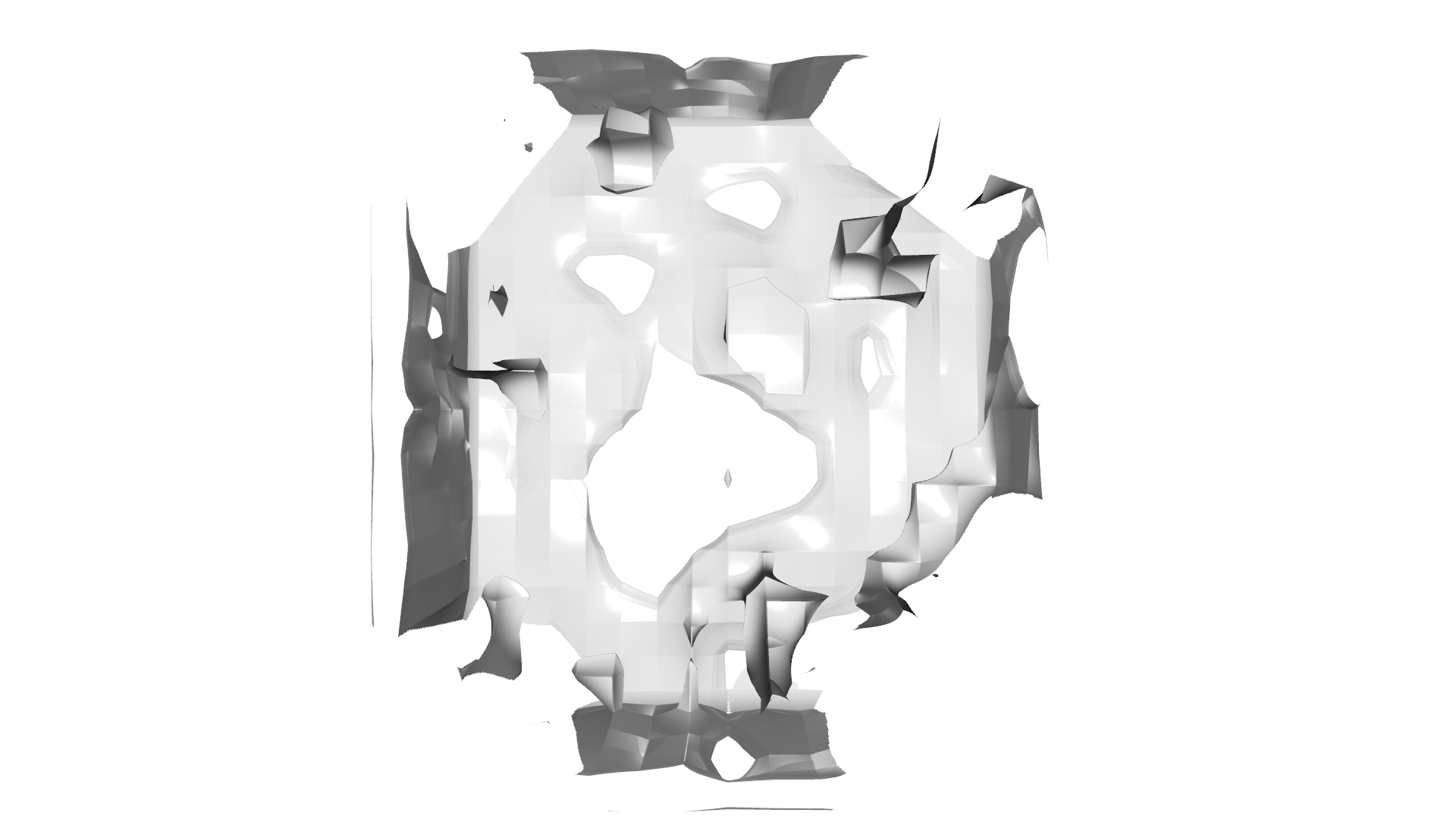}
        \caption{Level 1  mesh}
    \end{subfigure}
    \\
    \begin{subfigure}[t]{0.24\linewidth}
        \centering
        \includegraphics[trim=0 0 0 0, clip, width=\linewidth]{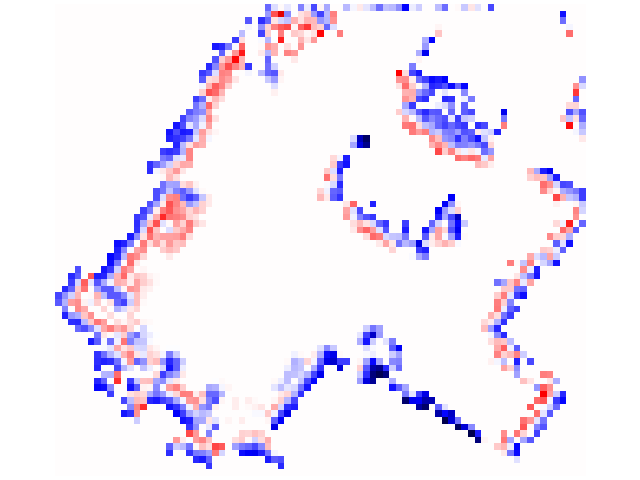}
        \caption{Level 2 $r^\text{in}_1$}
    \end{subfigure}
    \begin{subfigure}[t]{0.24\linewidth}
        \centering
        \includegraphics[trim=0 0 0 0, clip, width=\linewidth]{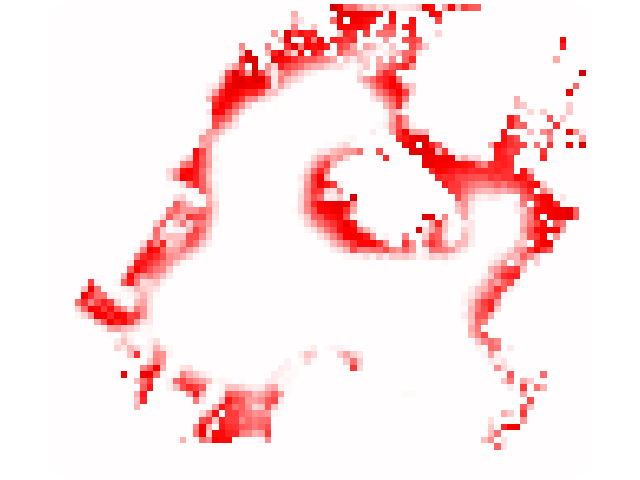}
        \caption{Level 2 $r^\text{in}_3$}
    \end{subfigure} 
    \begin{subfigure}[t]{0.22\linewidth}
        \centering
        \includegraphics[trim=0 0 0 0, clip, width=\linewidth]{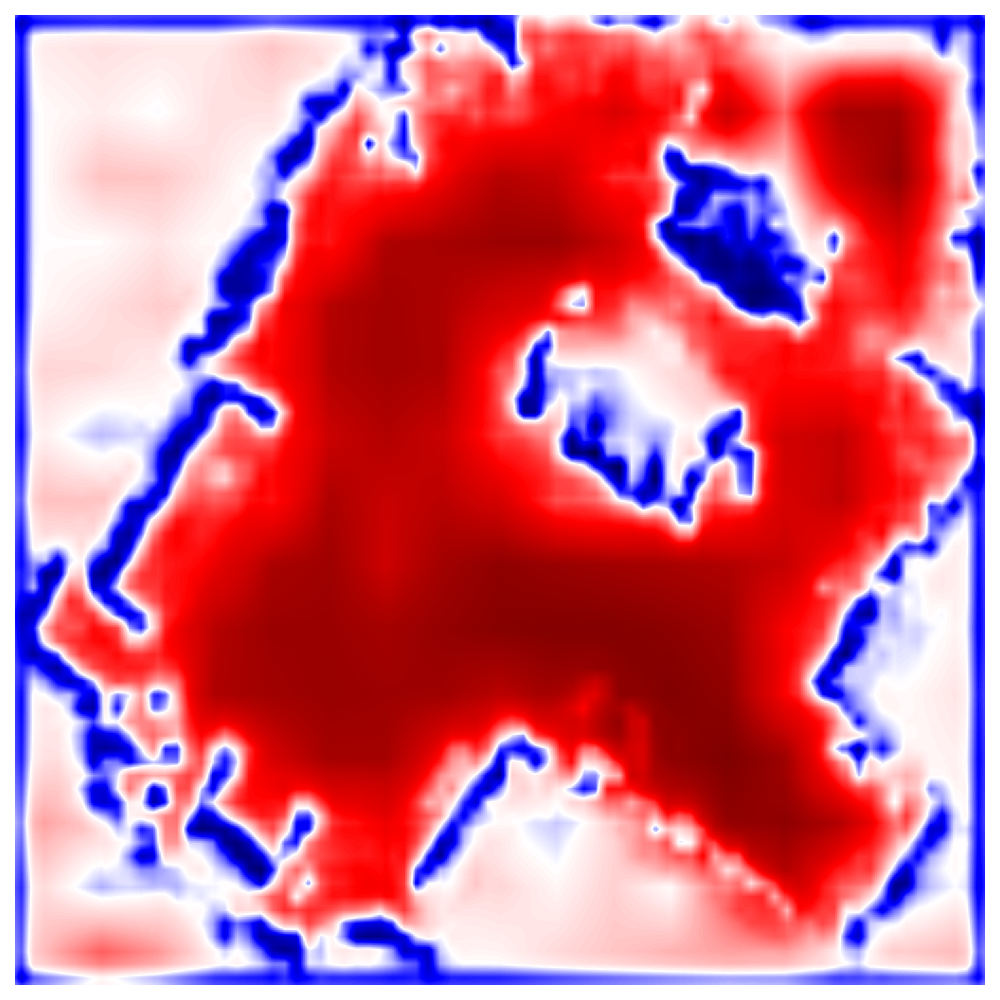}
        \caption{Level 2 SDF}
    \end{subfigure}
    \begin{subfigure}[t]{0.24\linewidth}
        \centering
        \includegraphics[trim=400 0 400 0, clip, width=\linewidth]{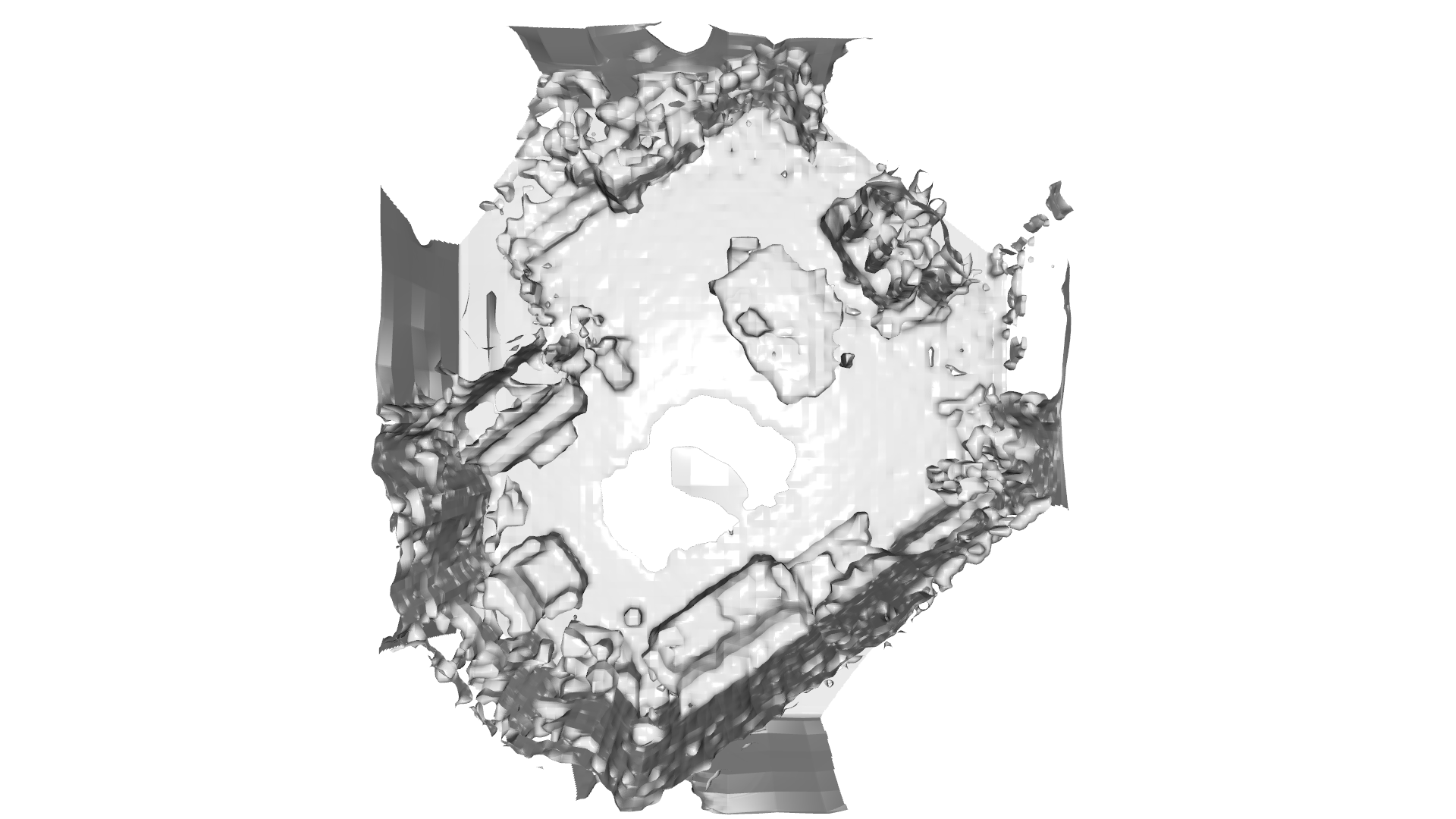}
        \caption{Level 2  mesh}
    \end{subfigure}
    \caption{Inputs and output predictions from learned hierarchical encoders on ScanNet scene 0024. Red and blue show positive and negative residual or SDF values, respectively. }
    \label{fig:encoder_io}
\end{figure}

%% file: sections/exp_scene_understanding.tex
\subsection{Scene understanding}
\label{sec:scene_understanding_exp}

\begin{figure*}[t]
    \centering
    \begin{subfigure}[t]{0.35\linewidth}
        \centering
        \includegraphics[width=\linewidth]{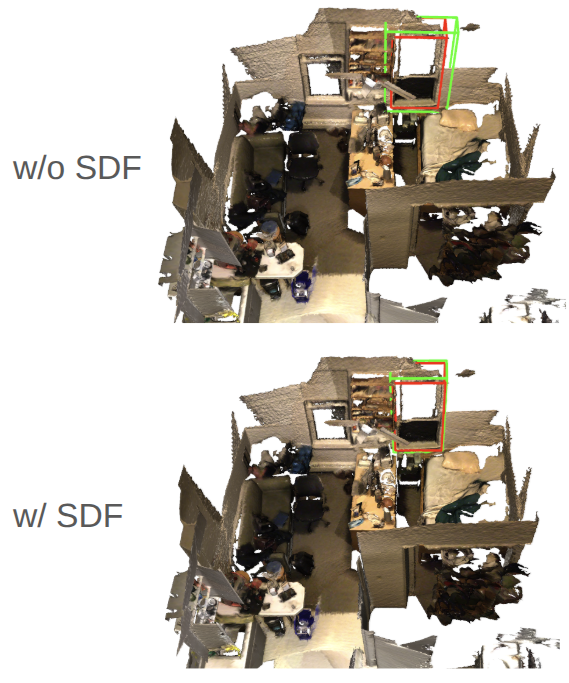}
        \caption{Query text: \textit{The window is in the back of the room behind the bed. It is to the right of the book shelf.}}
        \label{fig:scene_understanding_ScanRefer}
    \end{subfigure}
    \begin{subfigure}[t]{0.33\linewidth}
        \centering
        \includegraphics[width=\linewidth]{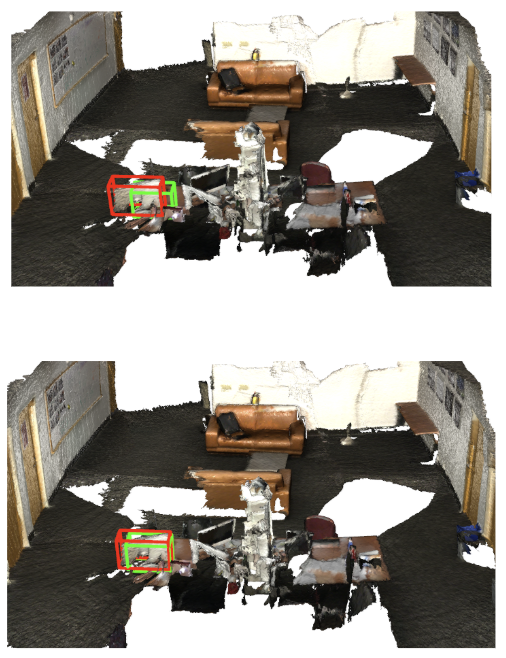}
        \caption{Query text: \textit{Find the monitor that is far from the shelf.}}
        \label{fig:scene_understanding_Sr3D}
    \end{subfigure}%
    \begin{subfigure}[t]{0.24\linewidth}
        \centering
        \includegraphics[width=\linewidth]{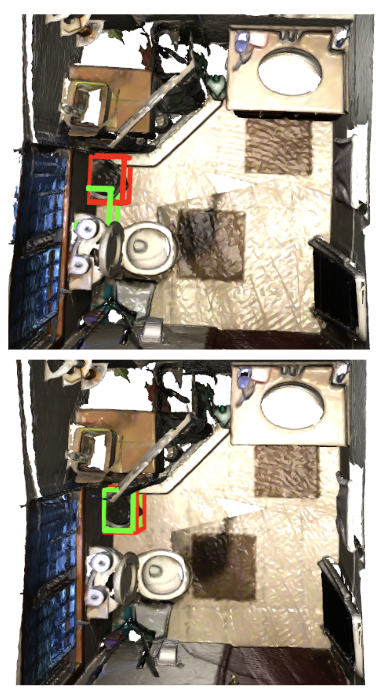}
        \caption{Query text: \textit{The trash can next to the toilet.}}
        \label{fig:scene_understanding_Nr3D}
    \end{subfigure}
    \caption{\textbf{Visualization of open-vocabulary object grounding results.}
    We visualize the results on different scenes from ScanRefer~\citep{chen2020scanrefer}, Sr3D~\citep{achlioptas2020referit3d}, and Nr3D~\citep{achlioptas2020referit3d}. {\textcolor{green}{Green boxes} are prediction} and {\textcolor{red}{red boxes} are ground truth}. Because the grid coordinates do not necessarily lie on the object surface, the model lacks detailed structural information, thus causing inaccurate object bounding box size and location. This causes accuracy drop with stricter IoU thresholds.}
    \label{fig:object_grounding_visualization}
\end{figure*}

\begin{table*}[t]
\centering
\small
\setlength{\tabcolsep}{2pt}
\renewcommand{\arraystretch}{1.08}
\caption{Comparison with prior methods on SR3D \citep{achlioptas2020referit3d}, NR3D \citep{achlioptas2020referit3d}, and ScanRefer \citep{chen2020scanrefer}. Map size denotes the memory required to store the scene representation. We show the results of our complete method in \textbf{bold}.}
\label{tab:scene_understanding}
\begin{tabular}{lccccccccc}
\toprule
& \multicolumn{2}{c}{Joint Evaluation} & \multicolumn{2}{c}{SR3D} & \multicolumn{2}{c}{NR3D} & \multicolumn{2}{c}{ScanRefer}  & \multicolumn{1}{c}{All Scenes} \\

\cmidrule(lr){2-3}
\cmidrule(lr){4-5}
\cmidrule(lr){6-7}
\cmidrule(lr){8-9}
\cmidrule(lr){10-10}

Method
& Acc@25 & Acc@50 & Acc@25 & Acc@50 & Acc@25 & Acc@50 & Acc@25 & Acc@50 & Map size (MB) $\downarrow$ \\
\midrule

\rowcolor{gray!20}
\multicolumn{10}{l}{\textit{Mesh PC}} \\

ReferIt3DNet~\citep{achlioptas2020referit3d}
& 26.6 & -- & 27.7 & -- & 24.0 & --  & 26.4 & 16.9 & -- \\
ScanRefer~\citep{chen2020scanrefer}
& -- & -- & -- & -- & -- & -- & 35.5 & 22.4 & -- \\
InstanceRefer~\citep{yuan2021instancerefer}
& 33.6 & -- & 31.5 & -- & 29.9 & -- & 40.2 & 32.9 & -- \\
BUTD-DETR~\citep{jain2022bottom}
& 50.28 & -- & 52.1 & -- & 43.3 & -- & 52.2 & 39.8 & -- \\
3D-VisTA~\citep{zhu20233d}
& 53.1 & 48.1 & 56.5 & 51.5 & 47.7 & 42.2 & 51.0 & 46.2 & -- \\

\midrule

\rowcolor{gray!20}
\multicolumn{10}{l}{\textit{Sensor PC + Proposals from Mesh PC}} \\

3D-VisTA~\citep{zhu20233d}
& 45.9 & 41.8 & 47.2 & 43.2 & 42.1 & 37.4 & 46.4 & 42.5 & -- \\
ConcreteNet~\citep{unal2024four} & -- & -- & -- & -- & -- & -- & 56.12 & 49.50 & -- \\

\midrule

\rowcolor{gray!20}
\multicolumn{10}{l}{\textit{Sensor PC}} \\

BUTD-DETR~\citep{jain2022bottom}
& 40.7 & 26.6 & 43.3 & 28.9 & 32.2 & 19.4 & 42.2 & 27.9 & -- \\
Llama VLM Baseline
& 28.8 & 18.3 & 21.3 & 13.9 & 28.0 & 16.9 & 37.3 & 24.2 & -- \\
GPT-4o VLM Baseline
& 38.6 & 25.5 & 29.2 & 18.9 & 38.2 & 25.1 & 48.2 & 32.5 & -- \\

\rowcolor{blue!13} Locate-3D w/ raw point cloud 
& 60.3 & 46.3 & 64.0 & 48.7 & 54.1 & 40.7 & 58.6 & 46.5 & 286.5 \\
\rowcolor{blue!13} Locate-3D w/ our 3D grids 
& 58.3 & 43.3 & 62.6 & 46.7 & 50.8 & 35.6 & 56.5 & 43.4 & 13.2 \\
\rowcolor{blue!13} Ours w/o SDF features 
& 49.3 & 28.4 & 51.6 & 29.9 & 41.2 & 21.7 & 51.5 & 31.0 & 12.6 \\
\rowcolor{blue!13} Ours w/ SDF features 
& \textbf{58.5} & \textbf{43.6} & \textbf{62.1} & \textbf{46.3} & \textbf{50.4} & \textbf{35.6} & \textbf{58.3} & \textbf{45.1} & \textbf{13.4} \\
\bottomrule
\end{tabular}

\end{table*}

We then evaluate the semantic understanding part and conduct experiments on the task of open-vocabulary object grounding, which requires localizing an object (in 3D bounding box) from a text description that may include a combination of attributes (e.g., “red backpack”) and/or spatial relationships (e.g., “the fridge near the brown table”). 

\myParagraph{Experiment setup} Following standard practice, we choose Sr3D~\citep{achlioptas2020referit3d}, Nr3D~\citep{achlioptas2020referit3d}, and ScanRefer~\citep{chen2020scanrefer} as our main datasets. To evaluate the object grounding accuracy, we compute the IoU between the predicted object bounding box and the ground-truth bounding box. The prediction is considered correct if the IoU is larger than or equal to a threshold and wrong if it is lower. We report the accuracy at both thresholds $0.25$ and $0.5$. Following Locate-3D~\citep{mcvay2025locate}, we report the Top-1 accuracy. Since \citet{mcvay2025locate} have not open-sourced their data and evaluation scripts, we pre-processed the data and implemented the evaluation script ourselves based on their descriptions, and used the same evaluation code for all models for a fair comparison.
We first show that our hierarchical 3D grid map is compatible with existing point-cloud-based methods by running the Locate-3D model~\citep{mcvay2025locate} with queried point features from the grid maps.  
Then we show the results of our SDF-guided object grounding model (\Cref{sec:sdf_guided_grounding}) and show the results without using the SDF features for an ablation study. 

\myParagraph{Results} The main results are shown in \Cref{tab:scene_understanding}. We can see that 
Our results show that our hierarchical 3D grid maps can achieve more than $20\times$ memory compression with only $1\% \sim 3\%$ overall accuracy loss. For the map size comparison, the reported memory of Locate-3D w/ our 3D grids includes both vision-language grid features and the query points; Ours w/ SDF features includes both SDF and vision-language features. We can see that by utilizing SDF features, the accuracy is significantly improved, especially for the accuracy with IoU threshold of $0.5$, i.e., over $15\%$. The reason is mentioned in \Cref{sec:sdf_guided_grounding}, the grid coordinates do not necessarily lie on the object surface and are distributed evenly in space, thus losing detailed structural information. The mismatch between grid coordinates and object positions causes inaccurate box size or position, which is more noticeable under a higher IoU threshold, i.e., $0.5$. Since SDF features indicate the distance to the surface, by utilizing them in our proposed way, we significantly compensate for the missing structural information and achieve similar results with running Locate-3D with our 3D grids. Some visualizations of our model without and with SDF features can be found in Figure~\ref{fig:object_grounding_visualization}.


%% file: sections/exp_pose_estimation.tex
\subsection{Pose estimation}

In this section, we evaluate the pose estimation with our hierarchical grids, including pose optimization in local submap and submap alignment. For this part of experiment, we keep using ScanNet~\citep{dai2017scannet} as our dataset.

\begin{table*}[!t]
\centering
\vspace{-0.1cm}
\caption{\edit{Comparison between local mapping and SLAM on ScanNet \citep{dai2017scannet} using colored ICP odometry as initial guess. We evaluate SLAM under both SDF-only supervision and joint SDF+CLIP supervision to assess the effect of CLIP features on SLAM accuracy. For each scene, we report translation RMSE (cm), rotation RMSE (deg), and Chamfer-L1 error (cm). Best results are highlighted in \textbf{bold}.}}
\label{tab:scannet_odometry}
\vspace{-0.1cm}
\resizebox{\textwidth}{!}{%
\begin{tabular}{|l|crr|crr|crr|crr|}
\hline
\multicolumn{1}{|c|}{Scene}  & \multicolumn{3}{c|}{0000} & \multicolumn{3}{c|}{0011} & \multicolumn{3}{c|}{0024} & \multicolumn{3}{c|}{0207} \\
\multicolumn{1}{|c|}{Method} & Tran err. $\downarrow$ & \multicolumn{1}{c}{Rot err. $\downarrow$} & \multicolumn{1}{c|}{C-l1 $\downarrow$} & Tran err. $\downarrow$ & \multicolumn{1}{c}{Rot err. $\downarrow$} & \multicolumn{1}{c|}{C-l1 $\downarrow$} & Tran err. $\downarrow$ & \multicolumn{1}{c}{Rot err. $\downarrow$} & \multicolumn{1}{c|}{C-l1 $\downarrow$} & Tran err. $\downarrow$ & \multicolumn{1}{c}{Rot err. $\downarrow$} & \multicolumn{1}{c|}{C-l1 $\downarrow$} \\ 
\hline

Color ICP~\citep{park2017colored} + Mapping & 
\multicolumn{1}{r}{40.86} & 10.17 & 15.04 & 
\multicolumn{1}{r}{17.25} & 5.35 & 8.8 &  
\multicolumn{1}{r}{18.98} & 5.04 & 10.29 &  
\multicolumn{1}{r}{19.47} & 6.25 & 12.83 \\ 

Color ICP~\citep{park2017colored} + SLAM (SDF) & 
\multicolumn{1}{r}{\textbf{17.33}} & \textbf{4.16} & 11.77 & 
\multicolumn{1}{r}{9.01} & \textbf{2.85} & 8.39 & 
\multicolumn{1}{r}{\textbf{9.12}} & \textbf{2.93} & \textbf{9.01} & 
\multicolumn{1}{r}{12.78} & \textbf{6.2} & 10.85 \\ 

Color ICP~\citep{park2017colored} + SLAM (SDF+CLIP) & 
\multicolumn{1}{r}{20.11} & 10.17 & 11.72 & 
\multicolumn{1}{r}{\textbf{8.16}} & 3.46 & \textbf{7.36} & 
\multicolumn{1}{r}{12.15} & 4.42 & 12.47 & 
\multicolumn{1}{r}{\textbf{11.14}} & 8.0 & \textbf{10.78} \\ 

\hline

\end{tabular}%
}
\vspace{-0.3cm}
\end{table*}

\myParagraph{Pose optimization in local submap experiment setup}
To evaluate joint mapping and pose optimization under more challenging initial trajectory estimates, we include an experiment with color ICP~\citep{park2017colored} as odometry estimation.
When running local SLAM, we let \AlgName incrementally process a sliding window of 10 frames with a stride of 10 and disable pose regularization in \eqref{eq:local_slam}.
We test pose optimization with both SDF features and vision-language, \ie, CLIP~\citep{clip} features. We follow the same cost functions for SDF and CLIP as introduced in \Cref{sec:local_mapping} and introduce pose as an optimizable variable as introduced in \Cref{sec:pose_optimizaion}. We follow the default grid size and feature dimension setting in \Cref{sec:grid_szie_and_feat_dim}.

\myParagraph{Pose optimization in local submap results} \Cref{tab:scannet_odometry} shows that, for SDF features, performing local SLAM substantially improves performance compared to mapping only. However, we found that CLIP features have very little effect in pose estimation, but we still included the results for completeness. Adding the CLIP features can either increase or decrease the translation errors, and tends to increase the rotation errors even in cases where translation or reconstruction quality improves. This suggests that camera pose estimation is tightly coupled with geometry estimation, while vision-language features lack geometric information and do not really contribute to camera pose tracking.

\begin{table}[t]
\centering
\renewcommand{\arraystretch}{1.3}
\caption{Submap alignment evaluation on ScanNet \citep{dai2017scannet} from initial errors of $5$~deg and $0.20$~m.
For each scene, we report the GPU time (sec), and the final submap rotation error (deg) and translation error (m) compared to ground truth. Results averaged over 10 trials. Best and second-best results are highlighted in \textbf{bold} and \underline{underline}.}
\label{tab:scannet_align}
\resizebox{\linewidth}{!}{%
\begin{tabular}{|l|rrr|rrr|}
\hline
\multicolumn{1}{|c|}{Scene} & \multicolumn{3}{c|}{0011 (159 poses, 4 submaps)} & \multicolumn{3}{c|}{0024 (227 poses, 5 submaps)} \\
\multicolumn{1}{|c|}{Methods} & \multicolumn{1}{c}{Time$\downarrow$} & \multicolumn{1}{c}{Rot err$\downarrow$} & \multicolumn{1}{c|}{Tran err$\downarrow$} & \multicolumn{1}{c}{Time$\downarrow$} & \multicolumn{1}{c}{Rot err$\downarrow$} & \multicolumn{1}{c|}{Tran err$\downarrow$} \\ \hline
MIPS \citep{tang2023mips} & 52.34 & 3.90 & 0.15 & 116.57 & 2.45 & 0.09 \\
VFPP \citep{zhai2024vox} & 80.01 & 3.21 & 0.13 & 138.03 & \textbf{1.86} & \textbf{0.07} \\
ICP \citep{choi2015robust} & - & 2.71 & 0.17 & - & 10.04 & 0.33 \\
\AlgName (SDF) & 12.15 & {\ul 1.81} & {\ul 0.07} & {\ul 19.20} & 2.25 & 0.09 \\
\AlgName (CLIP) & \textbf{6.29} & 3.66 & 0.11 & \textbf{17.52} & 3.39 & 0.13 \\
\AlgName (SDF+CLIP) & 21.22 & \textbf{1.22} & \textbf{0.05} & 32.96 & {\ul 1.98} & {\ul 0.08} \\ \hline
\end{tabular}%
}
\end{table}

\myParagraph{Submap alignment experiment setup} The next set of experiments evaluates the submap alignment and fusion approach  in \AlgName.
We use the same method in MIPS-Fusion \citep{tang2023mips} to obtain the submap splits.
We select scene 0011 and perform local SLAM within all submaps starting from noisy poses with $1$~deg and $0.1$~m errors.
We then evaluate the performance of the proposed alignment (\Cref{alg:hier_alignment}) and the baseline \edit{MIPS \citep{tang2023mips}, VFPP~\citep{zhai2024vox}, and ICP~\citep{choi2015robust}} under increasing submap alignment errors.
All methods are run for 100 iterations before evaluating their results.
For \AlgName, we allocate 45 iterations for alignment at each feature level and 10 iterations for final alignment using SDF or vision-language feature prediction, which corresponds to setting $k_{f,1}=k_{f,2}=45$ and $k_s=10$ in \cref{alg:hier_alignment}.
\edit{All neural methods} use the same trust-region-based pose regularization, where the trust-region radius $\tau$ in \eqref{eq:trust_region} is set based on the initial submap alignment error.
For each setting of the initial error shown on the $x$ axis, we perform 10 random trials and get the final results. We also use the default grid size and feature dimension setting in \Cref{sec:grid_szie_and_feat_dim}.

\myParagraph{Submap alignment results} As shown in boxplots in \Cref{fig:submap_alignment}.
With a small initial alignment error of $1$~deg and $0.1$~m, all methods produce accurate alignment results.
However, as the initial error increases, both MIPS and VFPP methods quickly fail, showing that solely relying on SDF prediction is very sensitive to the initial guess.
\edit{The ICP baseline results in higher variance indicating more frequent alignment failures despite the use of robust optimization.}
The proposed hierarchical alignment in \AlgName is more robust under more severe misalignments and outperforms the baseline methods significantly.
\edit{Since our experiments are offline, we record the total GPU time and final accuracy. 
Timing results for ICP are omitted as the method is not implemented on GPU.
\Cref{tab:scannet_align} reports the results on scenes 0011 and 0024.}
The initial submap alignment error is set to $5$~deg and $0.2$~m, and results are collected over 10 trials.
\AlgName is significantly faster since the majority of iterations in \cref{alg:hier_alignment} only uses features for alignment and does not need to decode the SDF values. In terms of accuracy, \AlgName demonstrates competitive performance and achieves either best or second-best results.
We also test submap alignment with CLIP features in both \Cref{fig:submap_alignment} and \Cref{tab:scannet_align}. Different from local pose estimation, we find that CLIP features can help improve the submap alignment, especially the rotation error. The difference from local pose optimization is that, during submap alignment, we only optimize the submaps' poses and keep submap features frozen, thus avoiding the complex dependency between pose tracking and latent feature mapping. 

\begin{figure}[!ht]
    \centering
    \begin{subfigure}[b]{\linewidth}
        \centering
        \includegraphics[width=1.0\linewidth]{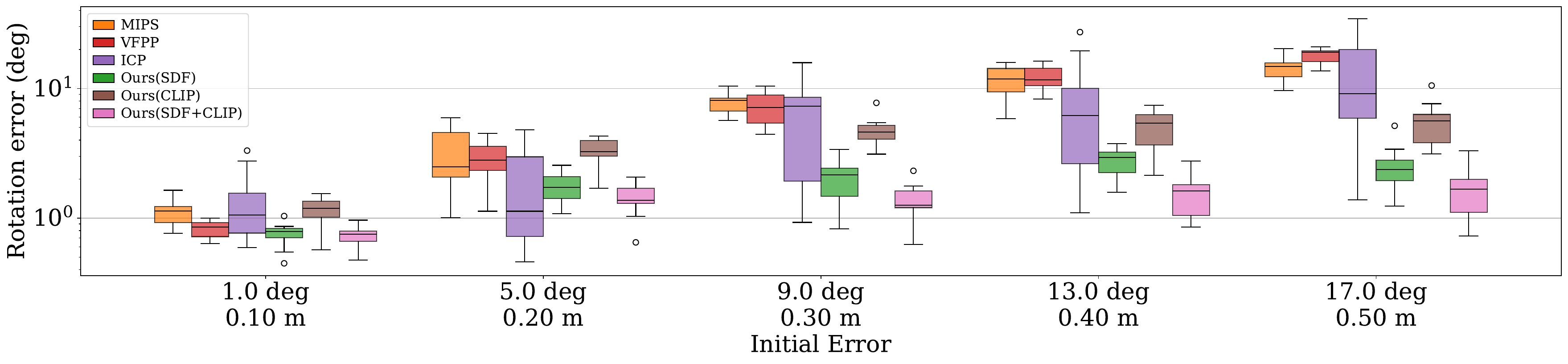}
        \captionsetup{justification=centering}
        \caption{Rotation error (deg)}
        \label{fig:sub_align_error_deg}
    \end{subfigure}\\
    \begin{subfigure}[b]{\linewidth}
        \centering
        \includegraphics[width=1.0\linewidth]{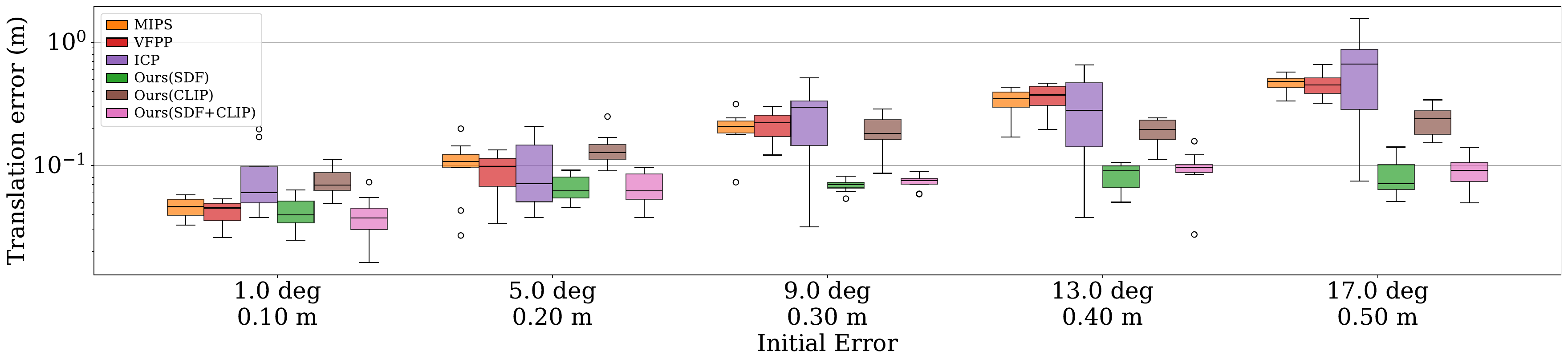}
        \captionsetup{justification=centering}
        \caption{Translation error (m)}
        \label{fig:sub_align_error_m}
    \end{subfigure}
    \caption{Submap alignment on ScanNet scene 0011 under varying initial errors over 10 trials for each configuration.}\label{fig:submap_alignment}
\end{figure}

%% file: sections/exp_outdoor.tex
\subsection{Evaluation on large-scale outdoor scenes}
\label{sec:experiments:newer_college}


\begin{table*}[t]
\centering
\renewcommand{\arraystretch}{1.3}
\caption{\edit{Evaluation on Newer College dataset \citep{zhang2021ncdmulti}.
For each scene, we report translation RMSE (m), rotation RMSE (deg), Chamfer-L1 error (cm), and F-score (\%) with a threshold of $20$ cm. Best and second-best results are highlighted in \textbf{bold} and \underline{underline}, respectively.}}
\label{tab:ncd}
\resizebox{0.8\textwidth}{!}{%
\begin{tabular}{|l|rrrr|rrrr|}
\hline
\multicolumn{1}{|c|}{Scene}  & \multicolumn{4}{c|}{Quad (1991 scans)}                                                                                                                                   & \multicolumn{4}{c|}{Maths Institute (2160 scans)}                                                                                                                        \\
\multicolumn{1}{|c|}{Method} & \multicolumn{1}{c}{Tran err. $\downarrow$} & \multicolumn{1}{c}{Rot err. $\downarrow$} & \multicolumn{1}{l}{C-l1 $\downarrow$} & \multicolumn{1}{l|}{F-score $\uparrow$} & \multicolumn{1}{c}{Tran err. $\downarrow$} & \multicolumn{1}{c}{Rot err. $\downarrow$} & \multicolumn{1}{l}{C-l1 $\downarrow$} & \multicolumn{1}{l|}{F-score $\uparrow$} \\ \hline
ICP \citep{zhou2018open3d} + \AlgName~Mapping                & 4.06                                       & 11.02                                     & 42.0                                  & 32.58                                   & 3.57                                       & 15.66                                     & 48.87                                 & 24.68                                   \\
KISS-ICP \citep{vizzo2023kiss} + \AlgName~Mapping           & \textbf{0.08}                              & {\ul 0.98}                                & \textbf{13.96}                        & {\ul 82.68}                             & 0.24                                       & 2.4                                       & \textbf{20.65}                        & {\ul 63.41}                             \\
PIN-SLAM \citep{pan2024pin}                    & {\ul 0.10}                                 & \textbf{0.97}                             & 14.74                                 & 78.85                                   & \textbf{0.15}                              & \textbf{1.37}                             & {\ul 20.66}                           & \textbf{70.18}                          \\
\AlgName (Odometry) & 0.47 & 1.58 & 18.53 & 68.41 & 0.25 & {\ul 1.57} & 23.28 & 57.78 \\
\AlgName (Full) & {\ul 0.10} & {\ul 0.98}                                & {\ul 14.01}                           & \textbf{83.0}                           & {\ul 0.22}                                 & {\ul 1.57}                                & 20.89                                 & 62.07                                   \\ \hline
\end{tabular}%
}
\vspace{-0.3cm}
\end{table*}

\begin{figure}[!ht]
    \centering
    \includegraphics[trim=0 100 0 30, clip,width=\linewidth]{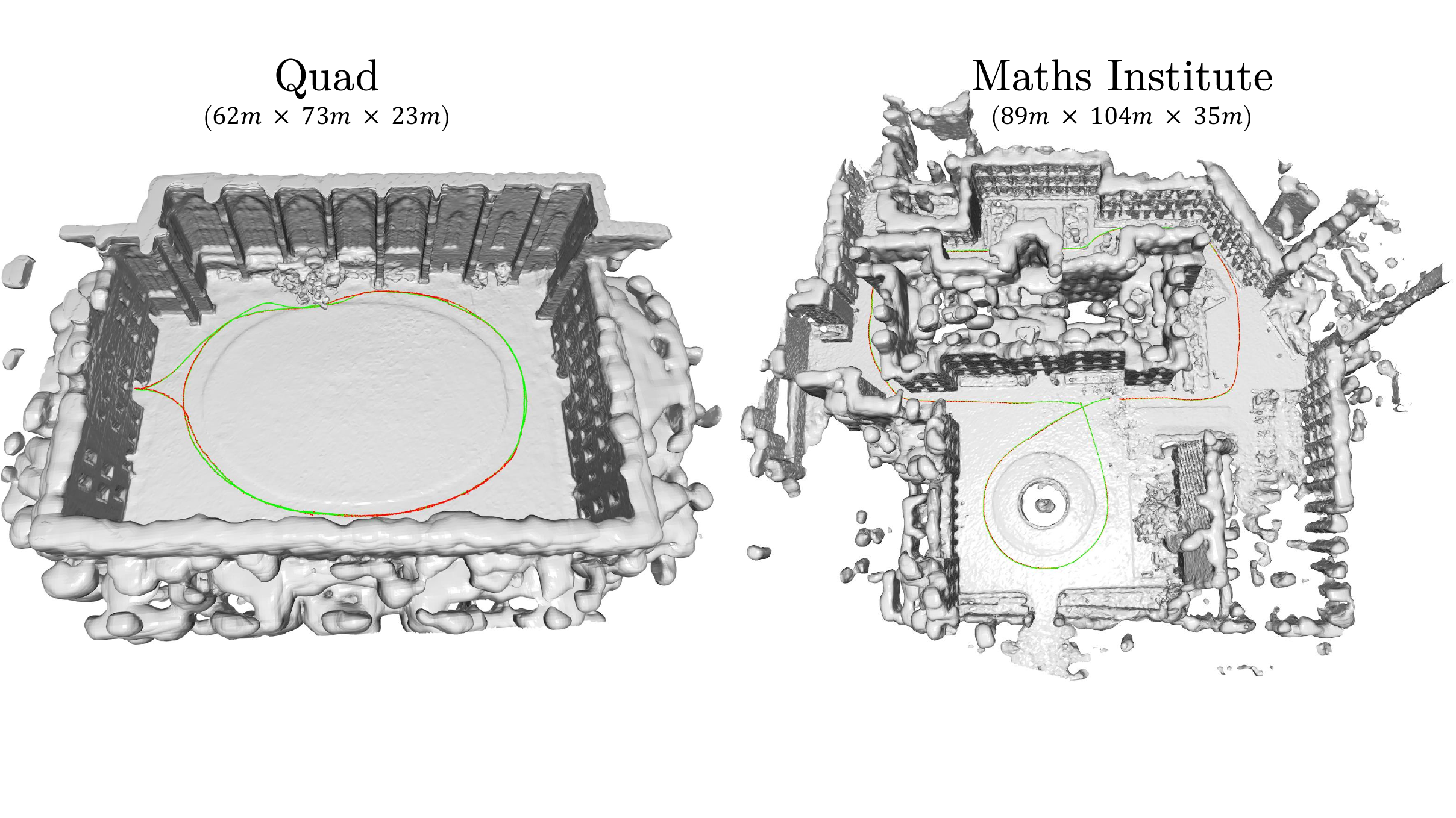}
    \caption{\edit{Qualitative evaluation on the Newer College dataset \citep{zhang2021ncdmulti} showing the final reconstructed meshes along with estimated and reference robot trajectories in red and green, respectively.}}
    \label{fig:ncd}
\end{figure}


In this section, we show \AlgName can be deployed in large-scale outdoor scenes. We focus on scene construction in this part and skip the object grounding due to the lack of ground-truth annotations. 

\myParagraph{Experiment setup} We choose the Quad-Easy and Maths-Easy sequences from the outdoor, \lidar-based Newer College Dataset \citep{zhang2021ncdmulti}.
For comparison, we run PIN-SLAM~\citep{pan2024pin}, KISS-ICP~\citep{vizzo2023kiss}, and point-to-point ICP in Open3D~\citep{zhou2018open3d} using their open-source implementations and default hyperparameters.
For KISS-ICP and point-to-point ICP, we obtain corresponding maps for evaluation by running our mapping pipeline with the pose estimates fixed.
For \AlgName, we first pre-train a decoder for each scene and then use the incremental implementation presented in \cref{sec:pose_optimizaion} to process each sequence.
Point-to-point ICP is used to initialize the tracking optimization at every scan, and 
a new submap is created every 400 scans.
We evaluate two variants of our method in this experiment.
The first (\method{Odometry}) only performs incremental tracking and mapping.
The second (\method{Full}) also performs submap alignment and fusion as described in \Cref{sec:global_optimization}.

\myParagraph{Results} The qualitative results are shown \Cref{fig:ncd}.
\Cref{tab:ncd} reports translation RMSE (m), rotation RMSE (deg), Chamfer L1 distance (cm), and F-score (\%).
In the Quad scene, our method achieves comparable performance with other state-of-the-art methods.
In the more challenging Maths Institute scene, \AlgName performs slightly worse than PIN-SLAM but still on par with KISS-ICP.
On both sequences, our results demonstrate that \AlgName substantially improves over the point-to-point ICP initialization. 
Further, the comparison between \method{Odometry} and \method{Full} validates the effectiveness of the proposed submap optimization approach to achieve global consistency.

%% file: sections/exp_ablation_studies.tex
\subsection{Ablation studies}
\label{sec:experiments:ablations}

We conclude the experiments with ablation studies to provide more insight into several design choices in \AlgName.

\myParagraph{Ablation on encoder and decoder pre-training}
The first ablation study investigates the performance improvements brought by pre-training the encoder and decoder networks in \AlgName.
For this, we use the same setup as the local mapping experiments in \cref{tab:scannet_mae}.
For each scene, we report the F-score (\%) calculated with a 5 cm threshold for SDF construction and cosine similarity for CLIP features. The results are presented at epochs 10 and 100 to demonstrate both early-stage and later-stage training performance.
In \Cref{tab:scannet_init}, we compare the following variants of \AlgName: 
\begin{itemize}
	\item \method{No-ED}: \Cref{alg:hier_local_mapping} without pre-trained decoder or encoder initialization. The grid features are initialized from a normal distribution with standard deviation $10^{-4}$, and are optimized jointly with the decoder.
    \item \method{No-E}: \Cref{alg:hier_local_mapping} using only pre-trained decoder and no encoder initialization. The grid features are initialized from a normal distribution with standard deviation $10^{-4}$.
    \item \method{Full}: \Cref{alg:hier_local_mapping} using both pre-trained decoder and encoder initialization.
\end{itemize}
As shown in \Cref{tab:scannet_init}, our hierarchical initialization scheme significantly speeds up training loss optimization, particularly when both the pre-trained encoder and decoder are used (see rows 1 and 3). Even with only the pre-trained decoder, optimization remains faster than training from scratch (rows 1 and 2). As expected, after sufficient training (\ie, at 100 epochs), all three methods converge in terms of training loss. Notably, using the pre-trained decoder yields better performance than training from scratch.
\edit{We attribute the slight degradation of \method{Full} compared to \method{No-E} in reconstruction F-score to domain mismatch since our encoders are pre-trained on synthetic Replica~\citep{straub2019replica} scenes only, while the vision-language cosine similarity difference in epoch 100 is neligible.
Overall, \cref{tab:scannet_init} still shows clear benefits of pre-trained encoders at early stages (10 epochs), which will be important for real-time estimation.} We do not report the F-score for \method{No-ED} at epoch 10 because it fails to produce a mesh due to insufficient training.

\begin{table}[t]
\centering
\renewcommand{\arraystretch}{1.6}
\caption{Ablation study on the effect of pre-trained encoder and decoder networks. For each scene, we report F-score (\%) for SDF construction and cosine similarity for vision-language feature at epochs 10 and 100. Results are averaged over five runs, and the best outcomes are highlighted in \textbf{bold}.}
\label{tab:scannet_init}
\begingroup
\resizebox{\linewidth}{!}{%
\begin{tabular}{|l|rr|rr|rr|rr|}
\hline
\multicolumn{1}{|c|}{Scene} & \multicolumn{2}{c|}{0011 \vspace{-0.25em}} & \multicolumn{2}{c|}{0011 \vspace{-0.25em}} & \multicolumn{2}{c|}{0207 \vspace{-0.25em}} & \multicolumn{2}{c|}{0207 \vspace{-0.25em}} \\ 
\multicolumn{1}{|c|}{} & \multicolumn{2}{c|}{\fontsize{6pt}{7pt}\selectfont (10 epoch) \vspace{-0.15em}} & \multicolumn{2}{c|}{\fontsize{6pt}{7pt}\selectfont (100 epoch) \vspace{-0.15em}} & \multicolumn{2}{c|}{\fontsize{6pt}{7pt}\selectfont (10 epoch) \vspace{-0.15em}} & \multicolumn{2}{c|}{\fontsize{6pt}{7pt}\selectfont (100 epoch) \vspace{-0.15em}}\\
\multicolumn{1}{|c|}{Method} 
& \multicolumn{1}{c}{F-score$\uparrow$} & \multicolumn{1}{c|}{cos sim.$\uparrow$}& \multicolumn{1}{c}{F-score$\uparrow$} & \multicolumn{1}{c|}{cos sim.$\uparrow$}
& \multicolumn{1}{c}{F-score$\uparrow$}& \multicolumn{1}{c|}{cos sim.$\uparrow$}& \multicolumn{1}{c}{F-score$\uparrow$} & \multicolumn{1}{c|}{cos sim.$\uparrow$}\\
\hline
{\method{No-ED}} 
& - & 76.4 & 59.3 & 92.9
& - & 74.1 & 58.5 & 89.5 \\
{\method{No-E}} 
& 66.9 & 92.4 & \textbf{77.7} & \textbf{95.5}
& 64.5 & 90.6 & \textbf{73.9} & \textbf{92.4} \\
{\method{Full}} 
 & \textbf{70.8} & \textbf{94.9} & 75.2 & 95.3 
 & \textbf{67.8} & \textbf{92.1} & 71.6 & 92.3 \\
\hline
\end{tabular}%
}
\endgroup
\end{table}

\myParagraph{Ablation on hierarchical alignment}
To investigate the effects of each alignment stage in \Cref{alg:hier_alignment}, we conduct ablation experiments on ScanNet scenes 0011 and 0207, each containing four submaps. We use the same setup as in the submap alignment experiments. In \cref{tab:hierarchical_alignment_ablation}, we compare the following variants of \AlgName:

\begin{itemize}
	\item \method{Coarse only}: \Cref{alg:hier_alignment} with only coarse grid ($l=1$) feature-based alignment.
	\item \method{Coarse+Fine}:  \Cref{alg:hier_alignment} with both coarse grid ($l=1$) and fine grid ($l=2$) feature-based alignment.
	\item \method{Full}: default \Cref{alg:hier_alignment} with both coarse grid ($l=1$) and fine grid ($l=2$) feature-based alignment, and a final alignment stage using predicted SDF values or CLIP features.
\end{itemize}
In \Cref{tab:hierarchical_alignment_ablation}, for SDF features, one can see that the full hierarchical approach (coarse, fine, and SDF) achieves the smallest rotation and translation errors, whereas using only the coarse alignment yields the largest errors. As expected, the full approach requires more computation time, as the last stage based on SDF alignment requires using the decoder to predict SDF values. For CLIP features, we do not observe significant improvement over coarse-only alignment after using fine-level features and final decoded features. Our hypothesis is that CLIP features are object-level (see \Cref{fig:feature_lifting}) and usually remain consistent in neighboring fine-level grids, causing relatively small differences between using coarse-level features and fine-level features.

\begin{table}[t]
    \centering
    \caption{
    Ablation study on hierarchical alignment using SDF and CLIP features on ScanNet
    from large initial errors of $10^\circ$ and $0.25$\,m.
    We report the alignment solver time (sec) and the final submap rotation error
    (deg) and translation error (m) compared to ground truth.
    Results are averaged over 10 trials. Best results are highlighted in \textbf{bold}.
    }
    \label{tab:hierarchical_alignment_ablation}
    
    \resizebox{\linewidth}{!}{
    \begin{tabular}{|ll|ccc|ccc|}
        \hline
        & & \multicolumn{3}{c|}{0011 (4 submaps)}
          & \multicolumn{3}{c|}{0207 (2 submaps)} \\
        \cline{3-8}
        Feature & Method
        & Time$\downarrow$ & Rot err$\downarrow$ & Tran err$\downarrow$ & Time$\downarrow$ & Rot err$\downarrow$ & Tran err$\downarrow$ \\
        \hline
        
        \multirow{3}{*}{SDF} & Coarse only & \textbf{1.20} & 4.73 & 0.16 & \textbf{0.25} & 4.86 & 0.07 \\
        & Coarse+Fine & 2.54 & 3.68 & 0.12 & 0.50 & 1.89 & 0.07 \\
        & Full & 10.13 & \textbf{2.71} & \textbf{0.09} & 1.68 & \textbf{1.11} & \textbf{0.05} \\

        \hline

        \multirow{3}{*}{CLIP} & Coarse only & \textbf{1.50} & \textbf{4.74} & 0.16
        & \textbf{0.28} & 5.03 & \textbf{0.12} \\
        & Coarse+Fine & 3.49 & 5.54 & 0.16 & 0.64 & 4.41 & 0.13 \\
        & Full & 4.06 & 4.95 & \textbf{0.15} & 0.72 & \textbf{4.10} & 0.13 \\

        \hline
    \end{tabular}
    }
\end{table}

%% file: sections/conclusion.tex
\section{Conclusion and future work}

In this work, we present \AlgName, a unified hierarchical 3D representation and optimization for joint geometric and semantic embedding. The hierarchical 3D grids together with the query and decoding functions enable latent mapping for different kinds of features and yield great memory efficiency. To further improve compute efficiency, we introduce techniques including hierarchical feature initialization with pre-trained encoders, submap decomposition, and hierarchical latent feature optimization for submap alignment. We demonstrate \AlgName can perform well in both scene construction and open-volcabuary object grounding. We explore how geometric features and semantic features can enhance each other, namely geometric features can facilitate semantic understanding by providing more fine-grained structure information while semantic features can help improve submap alignment by enforcing feature consistency. We conduct sufficient experiments on both indoor and outdoor datasets to demonstrate the effectiveness of our method. Future works may include leveraging the scene understanding capability for task planning like navigation, loco-manipulation, and 3D outdoor scene understanding.

%% file: sections/appendix_proof.tex
\section{Proofs}
\label{sec:proof}

\begin{proof}[Proof of \Cref{lem:linear_case}]
Without loss of generality, we assume the features are represented as a vector $F_l \in \Real^{|F_l|}$.
In the following, all points are expressed in the coordinate frame of the submap.
First, consider a single observed point $\bfx_j \in \Real^3$ and its label $y_j \in \Real$.
At each level $l$, the output feature 
$f_l(\bfx_j)$ in \Cref{def:grid} is a linear function of the features $F_l$ at this level.
This means that there exists a matrix $K_l(\bfx_j) \in \Real^{d \times |F_l|}$ such that
$f_l(\bfx_j) = K_l(\bfx_j) F_l$.
When the decoder $D_\theta$ is a linear function, the final output from the multiresolution feature grid is a weighted sum of features from all levels, \ie,
\begin{equation}
\begin{aligned}
h(\bfx_j;F)
&=
D_\theta \bigl(
	\oplus_{l \in [L]}
	f_l(\bfx_j)
\bigr)  \\
&=
\sum_{l=1}^L \theta_l^\top \! f_l(\bfx_j) 
= 
\sum_{l=1}^L  \theta_l^\top \! K_l(\bfx_j) F_l.
\end{aligned}
\end{equation}
We use $l$ to specifically refer to the target level that we seek to initialize in \eqref{eq:level_local_mapping}.
The measurement residual is,
\begin{equation}
h(\bfx_j;F) - y_j \!\!=\!\! 
\underbrace{\sum_{l'=1}^{l-1}\! \theta_{l'}^\top K_{l'}(\bfx_j) F_{l'} \!-\! y_j}_{r_{1:l-1}(\bfx_j)}
+ \theta_{l}^\top \! K_{l}(\bfx_j) F_{l},
\end{equation}
where we have used the fact that features at levels greater than $l$ are zero.
We can concatenate the above equation over all observations $\bfx_1, \hdots, \bfx_N$ as follows,
\begin{equation}
h(\bfx;F) - y = 
\underbrace{
\begin{bmatrix}
r_{1:l-1}(\bfx_1) \\
\vdots \\
r_{1:l-1}(\bfx_N)
\end{bmatrix} 
}_{r_{1:l-1}(\bfx)}
+ 
\underbrace{
\begin{bmatrix}
\theta_{l}^\top K_{l}(\bfx_1) \\
\vdots \\
\theta_{l}^\top K_{l}(\bfx_N)
\end{bmatrix}
}_{J}
F_l.
\end{equation} 
Note that the matrix $J$ is exactly the Jacobian of our model $h(\bfx; F)$ with respect to the level-$l$ features $F_l$. 
Substituting the above equation into \eqref{eq:level_local_mapping} shows that the problem is equivalent to the following linear least squares,
\begin{equation}
\underset{{F_l}}{\min} \norm{h(\bfx; F) - y}^2_2 = \norm{r_{1:l-1}(\bfx) + J F_l}^2_2.
\end{equation}
The solution is obtained by solving the normal equations, yielding the final expression in 
\eqref{eq:level_local_mapping_closed_form}.
\end{proof}

%% file: sections/appendix_details.tex
\section{Implementation and Training Details}
\label{sec:appendix_details}


\subsection{Map decoders}
\label{sec:map_decoders}
We introduce the architecture and training details of decoders for both SDF and vision-language features in this subsection.

\myParagraph{SDF decoder}
The decoder network $D^{\text{sdf}}_\theta$ is a single-layer MLP with hidden dimension 64.
To train the decoder network $D^{\text{sdf}}_\theta$ offline, we compile an offline training dataset using six scenes from Replica \citep{straub2019replica}. 
Within each scene $s \in S$, we randomly simulate 128 camera views with ground truth pose information. 
During training, all scenes share a single set of decoder parameters $\theta$, 
and each scene $s$ has its own dedicated grid features $F^{s}$.
We perform joint optimization using the SDF loss in \Cref{prob:local_mapping},
We use Adam \cite{kingma2014adam} with a learning rate of $10^{-3}$ and train for a total of 1200 epochs. 
During training, we employ a coarse-to-fine strategy where all fine level features are activated after 200 epochs. 
After training completes, the grid features $F^s$ are discarded and only decoder parameters $\theta$ is saved.

\myParagraph{Vision-language decoder}
The decoder network $D^{\text{vl}}_\theta$ is a five-layer MLP with hidden dimensions 256, 256, 512, 512, 768. To train this decoder, we pre-process all the 1114 training scenes in Sr3D~\citep{achlioptas2020referit3d}, Nr3D~\citep{achlioptas2020referit3d}, and ScanRefer~\citep{chen2020scanrefer} following the same way as Locate-3D~\citep{mcvay2025locate} so we can get the point and feature set $\{X^k\}_k \ \{Y^k\}_k$ (see \Cref{fig:feature_lifting} for an illustration). Same as training the SDF decoder, we perform joint optimization in \Cref{prob:local_mapping} but using the vision-language feature loss. We use Adam \cite{kingma2014adam} with a learning rate of $10^{-3}$. To keep memory usage manageable, we process 100 scenes at a time, training for 40 epochs per chunk with batches of 10 scenes and up to 30,000 sampled points per scene. We then save the scene features and optimization states to disk, and move to the next batch of scenes. All scenes share a single set of decoder parameters $\theta$, and each scene $s$ has its own dedicated grid features $F^{s}$. 
The obtained grid features $F^s$ can then be used to train our object grounding network.

\subsection{Map initialization encoders}
\label{sec:map_encoders}

We introduce the architecture and training details for the map initialization encoders for both SDF and vision-language features in this subsection.

\myParagraph{SDF map initialization encoder} For the SDF encoder network $E^{\text{sdf}}_{\phi_l}$ at each level $l$, we first use a small 3D CNN to process the input voxelized features.
Each CNN has 2 hidden layers, a kernel size of 3, and the number of feature channels doubles after every layer.
The CNN outputs at all 3D vertices are processed by a shared two-layer MLP with hidden dimension 16 to predict the target feature grid $F_l$.
We use ReLU as the nonlinear activation for all networks. The offline training of encoder networks follows a similar setup and uses the same Replica scenes.
The design of the training loss is presented in \cref{sec:local_mapping} in the main paper.
We sequentially train the encoder networks from coarse to fine levels.
At level $l$, we use the trained encoders from previous levels to compute $F^s_{1:l-1}$, which are needed to evaluate the training loss in \eqref{eq:encoder_training_loss}.
To account for noisy pose estimates at test time, we simulate pose errors  with $1$~deg and $1$~cm  rotation and translation standard deviations. 
Each encoder is trained using Adam \cite{kingma2014adam} with a learning rate of $10^{-3}$ for 1000 epochs.

\myParagraph{Vision-language map initialization encoder}
For the vision-language initialization encoder $E^{\text{vl}}_{\phi_l}$ at each level $l$,
we first use a six-layer, 512-channel MLP with BatchNorm and ReLU to embed the raw point features into a latent space. Each point's feature is distributed to its eight neighboring grid vertices using trilinear weights. Those features at each grid vertex are aggregated by max pooling. The pooled features are then processed by a 256-channel FPT~\citep{park2022fast} attention block, and a two-layer MLP (256→256→64) to predict latent grid features. 
To predict the raw point features, we trilinearly interpolate the predicted latent grid features in both levels, concatenate them into a 128-dimensional vector, and pass it through the pretrained decoder. We use the same cosine similarity loss as mapping (\Cref{sec:local_mapping}) to train the encoder. The decoder remains frozen during training. We train the encoders for both levels jointly and use a loss weight of 0.25 for prediction based on coarse-level features and a loss weight of 1 for prediction based on both coarse- and fine-level features. We train the encoder for 3000 epochs using Adam with a learning rate of $10^{-3}$.

\subsection{Object-grounding model}
\label{sec:object_grounding_model_details}

We follow the same number of attention blocks and the same latent feature dimensions in the encoder of Locate-3D~\citep{mcvay2025locate} but use the FPT~\citep{park2022fast} architecture. Locate-3D has two stages of training: 3D-JEPA self-supervised pre-training of the encoder and joint encoder-decoder pose-training. We skip the pretraining stage by distilling the 3D-JEPA features to our encoder (either with or without SDF features). More specifically, we run the 3D-JEPA encoder on the original Locate-3D point clouds and interpolate the encoded features to the 3D grids. Then we supervise our encoders by computing the cosine similarity loss between our encoder output and the interpolated 3D-JEPA features. In this distillation stage, the encoders are trained with batch size 128, learning rate $10^{-3}$, and 2000 epochs. For each scene, we randomly sample maximal $30000$ points at each iteration. Then, we connect our encoders with the trained Locate-3D decoder and post-train them jointly. We do this mainly to skip the tedious pre-training process, and this does not affect conveying our main ideas. In fact, the accuracy without post-training is very low, \ie, $\text{Acc@25}=0.19$ and $\text{Acc@50}=0.08$. For the post-training of the object grounding network, we train it on all the annotations of the three datasets (Sr3D~\citep{achlioptas2020referit3d}, Nr3D~\citep{achlioptas2020referit3d}, and ScanRefer~\citep{chen2020scanrefer}) for 20 epochs with a learning rate of $3 \times 10^{-5}$ and a batch size of 48. We perform stochastic weight averaging~\citep{izmailov2018averaging} in the last 10 epochs to get the final model weights.